\documentclass[twoside,11pt]{article}

\usepackage{jmlr2e}
\usepackage{amsmath}
\usepackage{mathtools}
\usepackage{color}
\usepackage{stmaryrd}
\usepackage{commath}
\usepackage{hyperref}
\usepackage{mdframed}
\usepackage{wrapfig}

\def\RR{{\mathbb{R}}}

\DeclareMathOperator*{\argmin}{arg\,min}

\def\XX{\mathcal{X}}
\def\EE{\mathbb{E}}

\usepackage[textwidth=1.850cm, textsize=scriptsize]{todonotes}

\jmlrheading{18}{2017}{1-35}{9/15; Revised 4/17}{6/17}{15-495}{Pedregosa, Bach and Gramfort}

\ShortHeadings{On the Consistency of Ordinal Regression Methods}{Pedregosa, Bach and Gramfort}
\firstpageno{1}

\begin{document}

\title{On the Consistency of Ordinal Regression Methods}

\author{\name Fabian Pedregosa \email f$@$bianp.net\\
      \addr INRIA \\
      D\'epartement d'informatique de l'ENS, \'Ecole normale sup\'erieure, CNRS, PSL Research University\\
      Paris, France       \AND
       \name Francis Bach \email francis.bach$@$ens.fr \\
       \addr INRIA \\
       D\'epartement d'informatique de l'ENS, \'Ecole normale sup\'erieure, CNRS, PSL Research University\\
      Paris, France       \AND
       \name Alexandre Gramfort \email alexandre.gramfort$@$inria.fr \\
       \addr LTCI, T\'el\'ecom ParisTech, Universit\'e Paris-Saclay \\
       INRIA, Universit\'e Paris-Saclay\\
       Saclay, France}

\editor{Tong Zhang}

\maketitle

\begin{abstract}
Many of the ordinal regression models that have been proposed in the literature can be seen as methods that minimize a convex surrogate of the zero-one, absolute, or squared loss functions. A key property that allows to study the statistical implications of such approximations is that of \emph{Fisher consistency}. Fisher consistency is a desirable property for surrogate loss functions and implies that in the population setting, i.e., if the probability distribution that generates the data were available, then optimization of the surrogate would yield the best possible model.
In this paper we will characterize the Fisher consistency of a rich family of surrogate loss functions used in the context of ordinal regression, including support vector ordinal regression, ORBoosting and least absolute deviation. We will see that, for a family of surrogate loss functions that subsumes support vector ordinal regression and ORBoosting, consistency can be fully characterized by the derivative of a real-valued function at zero, as happens for convex margin-based surrogates in binary classification. We also derive excess risk bounds for a surrogate of the absolute error that generalize existing risk bounds for binary classification. Finally, our analysis suggests a novel surrogate of the squared error loss. We compare this novel surrogate with competing approaches on 9 different datasets. Our method shows to be highly competitive in practice, outperforming the least squares loss on 7 out of 9 datasets.
\end{abstract}

\section{Introduction}

In ordinal regression the goal is to learn a rule to predict labels from an ordinal scale, i.e., labels from a discrete but ordered set. This arises often when the target variable consists of human generated ratings, such as (``do-not-bother'' $\prec$ ``only-if-you-must'' $\prec$ ``good'' $\prec$ ``very-good'' $\prec$ ``run-to-see'') in movie ratings~\citep{crammer2001pranking},  (``absent'' $\prec$ ``mild'' $\prec$ ``severe'') for the symptoms of a physical disease~\citep{ARMSTRONG01011989} and the NRS-11 numeric rating scale for clinical pain measurement~\citep{PAPR:PAPR3034}. Ordinal regression models have been successfully applied to fields  as diverse as econometrics~\citep{greene1997econometric}, epidemiology~\citep{ananth1997regression}, fMRI-based brain decoding~\citep{doyle2013multivariate} and collaborative filtering~\citep{Rennie}.

Ordinal regression shares properties--and yet is fundamentally different--from both multiclass classification and regression. As in the multiclass classification setting, the target variables consist of discrete values, and as in the regression setting (but unlike the multiclass setting) there is a meaningful order between the classes. If we think of the symptoms of a physical disease, it is clear that if the true label is ``severe'' it is preferable to predict ``mild'' than ``absent''. Ordinal regression models formalize this notion of order by ensuring that predictions farther from the true label incur a greater penalty than those closer to the true label.

The ordinal regression approach also shares properties with the learning-to-rank problem~\citep{liu2011learning}, in which the goal is to predict the relative order of a sequence of instances. Hence, this approach focuses on predicting a relative order while ordinal regression focuses on predicting a label for each instance. In this sense, it is possible for a ranking model (but not for an ordinal regression one) that predicts the wrong labels to incur no loss at all, as long as the relative order of those labels are correct, e.g. if the prediction is given by the true label plus an additive bias. Although ordinal regression and ranking are different problems, the distinction between both has not always been clear, generating some confusion. For example, in the past some methods presented with the word ``ranking'' in the title would be considered today ordinal regression methods~\citep{crammer2001pranking, Shashua, crammer2005online} and likewise some of the first pairwise ranking methods~\citep{herbrich1999} featured the word ordinal regression in the title.

Despite its widespread applicability, there exists a relative paucity in the understanding of the theoretical properties behind ordinal regression methods, at least compared to that of binary and multiclass classification. One such example is the notion of \emph{Fisher consistency}, which relates the minimization of a given loss to the minimization of a surrogate with better computational properties. The importance of this property stems from the fact that many supervised learning methods, such as support vector machines, boosting and logistic regression for binary classification, can be seen as methods that minimize a convex surrogate on the 0-1 loss. Such results have emerged in recent years for classification~\citep{Bartlett2003,Zhang,Tewari2007}, ranking~\citep{Duchi2010, calauzenes2012non}, structured prediction~\citep{ciliberto2016consistent,osokin2017structured} and multiclass classification with an arbitrary loss function~\citep{Ramaswamy2012,ramaswamy2014convex}, a setting that subsumes ordinal regression. Despite these recent progress, the Fisher consistency of most surrogates used within the context of ordinal regression remains elusive.  The aim of this paper is to bridge the gap by providing an analysis of Fisher consistency for a wide family of ordinal regression methods that parallels the ones that already exist for other multiclass classification and ranking.

\hfill

{\bfseries Notation}.
Through the paper, we will use $k$ to denote the number of classes (i.e., labels) in the learning problem. We will denote by $\mathcal{S}$ the subset of $\RR^{k-1}$ for which the components are non-decreasing, that is,
$$
\mathcal{S} := \left\{\alpha:  \alpha \in \RR^{k-1} \text{ and } \alpha_i \leq \alpha_{i+1} \text{ for } 1 \leq i \leq k-2 \right\} \quad.
$$
$\Delta^p$ denotes the $p$-dimensional simplex, defined as
$$
\Delta^p := \left\{ x \in \RR^p : x_i \geq 0 \text{ and } \sum_{i=1}^p x_i = 1 \right\} \quad.
$$
Following~\citet{knuth1992two} we use the Iverson bracket $\llbracket \cdot \rrbracket$ as
\begin{equation*}
\llbracket q \rrbracket :=
\begin{cases}
1  \text{ if q is true }\\
0 \text{ otherwise }\quad.
\end{cases}
\end{equation*}
We will also make reference to loss functions commonly used in binary classification. These are the hinge loss ($\varphi(t) = \max(1 - t, 0)$), the squared hinge loss ($\varphi(t) = \max(1-t, 0)^2$), the logistic loss ($\varphi(t) = \log(1 + e^{-t})$), exponential loss ($\varphi(t) = e^{-t}$) and the squared loss ($\varphi(t) = (1 - t)^2$).

\subsection{Problem setting}
Here we present the formalism that we will be using throughout the paper. Let $(\XX, \mathcal{A})$ be a measurable space. Let $(X, Y)$ be two random variables with joint probability distribution $P$, where $X$ takes its values in $\XX$ and $Y$ is a random label taking values in a finite set of $k$ \emph{ordered categories} that we will denote $\mathcal{Y} = \{1, \ldots, k\}$. In the ordinal regression problem, we are given a set of $n$ observations $\{(X_1,
Y_1), \ldots, (X_n, Y_n) \}$ drawn i.i.d.~from $X\times Y$ and the goal is to learn from the observations a measurable mapping called a~\emph{decision function} $f: \XX \rightarrow \mathcal{S} \subseteq \RR^{k-1}$ so that the \emph{risk} given below is as small as possible:
\begin{equation}
  \label{eq:l_risk}
  \mathcal{L}(f) := \EE(\ell(Y, f(X)))\quad,
\end{equation}
where $\ell: \mathcal{Y} \times \mathcal{S}$ is a \emph{loss function} that measures the disagreement between the true label and the prediction. For ease of optimization, the decision function has its image in a subset of $\RR^{k-1}$, and
the function that converts an element of $\mathcal{S}$ into a class label is called a \emph{prediction function}. The prediction function that we will consider through the paper is given for $\alpha \in \mathcal{S}$ by the number of coordinates below zero plus one, that is,
\begin{equation}\label{eq:pred}
\text{pred}(\alpha) := 1 + \sum_{i=1}^{k-1}\llbracket \alpha_i < 0 \rrbracket \quad .
\end{equation}
Note that for the case of two classes $\mathcal{Y} = \{1, 2\}$, the decision function is real-valued and the prediction defaults the common binary classification rule in which prediction depends on the sign of this decision function.

Different loss functions can be used within the context of ordinal regression. The most commonly used one is the absolute error, which measures the absolute difference between the predicted and true labels. For $\alpha \in \mathcal{S}$, this is defined as
\begin{equation}\label{eq:absolute_loss}
\ell(y, \alpha) := \abs{y - \text{pred}(\alpha)} \quad.
\end{equation}
The absolute error loss is so ubiquitous in ordinal regression that some authors refer to it simply as \emph{the} ordinal regression loss~\citep{Agarwal2008,Ramaswamy2012}. For this reason we give special emphasis on this loss. However, we will also describe methods that minimize the 0-1 loss (i.e., the classification error) and in Section \ref{sct:extension_other_loss} we will see how some results can be generalized beyond these and to general loss functions that verify a certain admissibility criterion.

In order to find the decision function with minimal risk it might seem appropriate to minimize Eq.~\eqref{eq:l_risk}. However, this is not feasible in practice for two reasons. First, the probability distribution $P$ is unknown and the risk must be minimized approximately based on the observations. Second, $\ell$ is typically discontinuous in its second argument, hence the empirical approximation to the risk is difficult to optimize and can lead to an NP-hard problem~\citep{feldman2012agnostic,ben2003difficulty}\footnote{Note that binary classification can be seen as a particular case of ordinal regression.}. It is therefore common to approximate $\ell$ by a function $\psi: \mathcal{Y} \times \mathcal{S} \to \RR$, called a \emph{surrogate loss function}, which has better computational properties. The goal becomes then to find the decision function that instead minimizes the \emph{surrogate risk}, defined as
\begin{equation}\label{eq:psi_risk}
\mathcal{A}(f) := \EE(\psi(Y, f(X))) \enspace.
\end{equation}

We are interested by the statistical implications of such approximation. Assuming that we have full knowledge of the probability distribution that generates the data $P$, what are the consequences of optimizing a convex surrogate of the risk instead of the true risk?

The main property that we will study in order to answer this question is that of  \emph{Fisher consistency}.
Fisher consistency is a desirable property for surrogate loss functions~\citep{Lin2004} and implies that in the population setting, i.e., if the probability distribution $P$ were available, then optimization of the surrogate would yield a function with minimal risk. From a computational point of view, this implies that the minimization of the  surrogate risk, which is usually a convex optimization problem and hence easier to solve than the minimization of the risk, does not penalize the quality (always in the population setting) of the obtained solution.

We will use the following notation for the optimal risk and optimal surrogate risk:
$$
\mathcal{L}^* := \inf_f \mathcal{L}(f) \quad \text{ and } \quad \mathcal{A}^* := \inf_f \mathcal{A}(f) \quad,
$$
where the minimization is done over all measurable functions $\mathcal{X} \to \mathcal{S}$. $\mathcal{L}^*$ is sometimes referred to as the \emph{Bayes risk}, and a decision function (not necessarily unique) that minimizes the risk is called a \emph{Bayes decision function}.

\hfill

We will now give a precise definition of Fisher consistency. This notion originates from a classical parameter estimation setting. Suppose that an estimator $T$ of some parameter $\theta$ is defined
as a functional of the empirical distribution $P_n$. We denote it $T(P_n)$. The estimator is said to be Fisher consistent if its population analog, $T(P)$, coincides with the parameter $\theta$. Adapting this notion to the context of risk minimization (in which the optimal risk is the parameter to estimate) yields the following definition, adapted from~\citet{Lin2004} to an arbitrary loss function $\ell$:

\begin{definition}({\bfseries Fisher consistency}) Given a surrogate loss function $\psi: \mathcal{Y}\times \mathcal{S} \to \RR$, we will say that the surrogate loss function $\psi$ is consistent with respect to the loss $\ell: \mathcal{Y}\times \mathcal{S} \to \RR$ if for every probability distribution over $X \times Y$ it is verified that every minimizer $f$ of the surrogate risk reaches Bayes optimal risk, that is,
$$\mathcal{A}(f) = \mathcal{A}^* \implies \mathcal{L}(f) = \mathcal{L}^* \quad .$$
\end{definition}

For some surrogates we will be able to derive not only Fisher consistency, but also \emph{excess risk bounds}. These are bounds of the form
$$
\gamma(\mathcal{L}(f) - \mathcal{L}^*) \leq \mathcal{A}(f) - \mathcal{A}^* \quad,
$$
for some real-valued function $\gamma$ with $\gamma(0) = 0$. These inequalities not only imply Fisher consistency, but also allow to bound the excess risk by the excess in surrogate risk. These inequalities play an important role in different areas of learning theory, as they can be used for example to obtain rates of convergence~\citep{Bartlett2003} and oracle inequalities~\citep{boucheron2005theory}.

\subsection{Full and conditional risk}\label{scs:full_conditional_risk}

The above definition of Fisher consistency is often replaced by a point-wise version that is easier to verify in practice. Two key ingredients of this characterization are the notions of \emph{conditional risk} and \emph{surrogate conditional risk} that we will now define. These are denoted by $L$ and $A$ respectively, and defined for any $\alpha \in \mathcal{S}$, $p \in \Delta^k$ by
\begin{equation}\label{eq:conditional_risk_def}
L(\alpha, p) :=  \sum_{i=1}^k p_i \ell(i, \alpha) \quad \text{ and } \quad
A(\alpha, p) := \sum_{i=1}^k p_i \psi(i, \alpha) \quad.
\end{equation}
The full and conditional risk are then related by the equations
$$
\begin{aligned}
\mathcal{L}(f) &= \EE_{X \times Y}(\ell(Y, f(X))) = \EE_{X}\EE_{Y|X}(\ell(Y, f(X))) = \EE_X(L(f(X), \eta(X))) \\
\mathcal{A}(f) &= \EE_{X \times Y}(\psi(Y, f(X))) = \EE_{X}\EE_{Y|X}(\psi(Y, f(X))) = \EE_X(A(f(X), \eta(X))) \quad ,
\end{aligned}
$$
where $\eta: \mathcal{X} \to \Delta^k$ is the vector of conditional probabilities given by $\eta_i(x) = P(y=i|X=x)$. As for the full risk, we will denote by $L^*$, $A^*$ the infimum of its value for a given $p \in \Delta^k$, i.e.,
$$
L^*(p) = \inf_{\alpha \in \mathcal{S}}L(\alpha, p) \quad\text{ and }\quad A^*(p) = \inf_{\alpha \in \mathcal{S}}A(\alpha, p) \quad.
$$

When the risk infimum over functions that can be defined independently at every $x \in \mathcal{X}$, it is possible to relate the minimization of the risk with that of the conditional risk since
\begin{equation}
\begin{aligned}\label{eq:infrisk}
\inf_{f} \mathcal{L}(f) &= \inf_f \EE_{X \times Y}\left(\ell(Y, f(X) )\right)  =
\EE_{X}  \left[ \inf_{f} \EE_{Y|X}(\ell(Y, f(X))) \right] \\
&= \EE_{X}  \left[ \inf_{\alpha} L(\alpha, \eta(X)) \right] \quad.
\end{aligned}
\end{equation}

This equation implies that the minimal risk can be achieved by minimizing pointwise the conditional risk $L(\cdot)$, which--in general--will be easier that direct minimization of the full risk. The condition for this, i.e., that the functions be estimated independently at every sample point, is verified by the set of measurable functions from the sample space into a subset of $\RR^{k}$ (in this case $\mathcal{S}$), which is the typical setting in studies of Fisher consistency. However, this is no longer true when inter-observation constraints are enforced (e.g. smoothness). As is common in studies of Fisher consistency, we will suppose that the function class verifies the property of Eq.~\eqref{eq:infrisk} and we will discuss in Section~\ref{scs:parametric_consistency} an important family of functions in which this requisite is not met.

We will now present a characterization of Fisher consistency based on the pointwise risk which we will use throughout the paper.
Equivalent forms of this characterization have appeared under a variety of names in the literature, such as classification calibration~\citep{Bartlett2003,Ramaswamy2012}, infinite sample consistency~\citep{Zhang2004} and proper surrogates~\citep{buja2005loss, gneiting2007strictly}.

\begin{lemma}[Pointwise characterization of Fisher consistency]\label{lemma:characterization_Fisher}
  Let $A$ and $L$ be defined as in Eq~\eqref{eq:conditional_risk_def}. Then $\psi$ is Fisher consistent with respect to $\ell$ if and only if for all $p \in \Delta^k$ it is verified that
  \begin{equation}\label{eq:pointwise_condition}
  A(\alpha, p) = A^*(p) \implies L(\alpha, p) = L^*(p) \quad .
  \end{equation}
\end{lemma}

\begin{proof}
Let $\mathcal{L}$ and $\mathcal{A}$ denote the expected value of $\ell$ and $\psi$, as defined in Equations ~\eqref{eq:l_risk} and~\eqref{eq:psi_risk} respectively.

$(\impliedby)$ We prove that Eq.~\eqref{eq:pointwise_condition} implies Fisher consistency. Let $f$ be such that $\mathcal{A}(f) = \mathcal{A}^*$. Then it is verified that
$$
\begin{aligned}
\mathcal{A}(f) - \mathcal{A}^* &= \EE_X(A(f(X), \eta(X)) - A^*(\eta(X))) = 0 \quad.
\end{aligned}
$$
The value inside the expectation is non-negative by definition of $A^*$. Since this is verified for all probability distributions over $X \times Y$, then it must be true that $A(f(x), \eta(x)) = A^*(\eta(x))$ for all $x \in \mathcal{X}$. By assumption $L(f(X), \eta(X)) = L^*(\eta(X))$. Hence the excess risk verifies
$$
\mathcal{L}(f) - \mathcal{L}^* = \EE_X(L(f(X), \eta(X)) - L^*(\eta(X))) = \EE(0) = 0 \quad.
$$
and so $\psi$ is Fisher consistent with respect to $\ell$.

$(\implies)$ We prove that Fisher consistency implies Eq.~\eqref{eq:pointwise_condition}. We do so by contradiction: first suppose that there exists a surrogate that is Fisher consistent but Eq.~\eqref{eq:pointwise_condition} is not verified and arrive to a contradiction. If Eq.~\eqref{eq:pointwise_condition} is not verified then there exists $\tilde{\alpha} \in \mathcal{S}$ and $\tilde{p} \in \Delta^{k}$ such that
$$
A(\tilde{\alpha}, \tilde{p}) = A^*(\tilde{p}) \text{ and } L(\tilde{\alpha}, \tilde{p}) > L^*(\tilde{p}) \quad.
$$
We now construct a probability distribution $(X, Y)$ such that the Fisher consistency characterization is not verified in order to arrive to a contradiction. For this, consider the probability distribution $P$ such that $\eta(x) = \tilde{p}$ for all $x \in\mathcal{X}$. Consider also $f: \mathcal{X} \to \mathcal{S}$, the mapping that is constantly $\tilde{\alpha}$. Then it is verified that
$$
\mathcal{A}(f) - \mathcal{A}^* = \EE_X(A(f(X), \eta(X)) - A^*(\eta(X))) = \EE_X(A(\tilde{\alpha}, \tilde{p}) - A^*(\tilde{p})) = 0 \quad,
$$
and so $\mathcal{A}(f) = \mathcal{A}^*$. Likewise, the excess risk verifies
$$\mathcal{L}(f) - \mathcal{L}^* = \EE_X(L(f(X), \eta(X)) - L^*(\eta(X))) = \EE_X(L(\tilde{\alpha}, \tilde{p}) - L^*(\tilde{p})) > 0$$
and so $\psi$ cannot be Fisher consistent with respect to $\ell$. This is a contradiction, and concludes the proof.

\end{proof}

\subsection{Summary of main results}

The main contribution of this paper is to characterize the Fisher consistency of a wide family of surrogate loss functions used for the task of ordinal regression. Contrary to known results for multiclass classification and ranking, where One-vs-All and RankSVM have been proven to be inconsistent, in the ordinal regression setting common surrogates such as ORSVM and proportional odds will be proven to be Fisher consistent. One of the most surprising results of this paper is that for a particular class of surrogates that verify a \emph{decomposability} property, it is possible to provide a characterization of Fisher consistency and excess risk bounds that generalize those known for convex margin-based surrogates (loss functions of the form $\varphi(Y f(X))$) in binary classification.

We will introduce the surrogate loss functions that we consider in Section~\ref{sec:problem_setting}. These will be divided between surrogates of the absolute error and surrogate of the 0-1 loss. We organize their study as follows:

\begin{itemize}
  \item  In Sections \ref{sct:absolute_error_surrogates} and~\ref{scs:squared_error} we characterize the {\bf
  Fisher consistency for surrogates of the absolute and squared error}. The surrogates that we consider in this section are the all threshold (AT), the cumulative link (CL), the least absolute deviation (LAD) and the least squares (LS). Besides Fisher consistency, a decomposability of the AT loss will allow us to provide excess risk bounds for this surrogate.

  \item In Section~\ref{scs:zero_one_surrogates} we characterize the {\bfseries Fisher consistency of the surrogates of the 0-1 loss}. For this loss, denoted immediate threshold (IT), its Fisher consistency will depend on the derivative at zero of a real-valued convex function.

  \item In Section~\ref{sct:extension_other_loss} we {\bfseries construct a surrogate for an arbitrary loss function} that verifies an admissibility condition. We name this surrogate generalized all threshold (GAT). This loss function generalizes the AT and IT loss functions introduced earlier. We will characterize the Fisher consistency of this surrogate.

  \item Turning back to one of the topics mentioned in the introduction, we discuss in Section~\ref{scs:parametric_consistency} the {\bfseries implications of inter-observational constraints in Fisher consistency}. Following~\citet{shi2015hybrid}, we define a restricted notion of consistency known as $\mathcal{F}$-consistency of parametric consistency and give sufficient conditions for the \mbox{$\mathcal{F}$-consistency} of two surrogates.

  \item In Section~\ref{scs:experiments} we {\bfseries examine the empirical performance of a novel surrogate}. This novel surrogate is a particular instance of the GAT loss function introduced in Section~\ref{sct:extension_other_loss} when considering the squared error as evaluation metric. We compare this novel surrogate against a least squares model on 9 different datasets, where the novel surrogate outperforms the least squares estimate on 7 out of the 9 datasets.

\end{itemize}

\subsection{Related work}

Fisher consistency of binary and multiclass classification for the zero-one loss has been studied for a variety of surrogate loss functions, see e.g.~\citep{Bartlett2003,Zhang,Tewari2007,Reid2010}. Some of the results in this paper generalize known results for binary classification to the ordinal regression setting. In particular,~\citet{Bartlett2003} provide a characterization of the Fisher consistency for convex margin-based surrogates that we extend to the all threshold (AT) and immediate threshold (IT) family of surrogate loss functions. The excess error bound that we provide for the AT surrogate also generalizes the excess error bound given in~\citep[Section 2.3]{Bartlett2003}.

Fisher consistency of arbitrary loss functions (a setting that subsumes ordinal regression) has been studied for some surrogates. \citet{lee2004multicategory} proposed a surrogate that can take into account generic loss functions and for which Fisher consistency was proven by~\citet{Zhang2004}. In a more general setting, \citet{Ramaswamy2012,ramaswamy2014convex} provide necessary and sufficient conditions for a surrogate to be Fisher consistent with respect to an arbitrary loss function. Among other results, they prove consistency of least absolute deviation (LAD) and an $\varepsilon$-insensitive loss with respect to the absolute error for the case of three classes ($k = 3$). In this paper, we extend the proof of consistency for LAD to an arbitrary number of classes. Unlike previous work, we consider the so-called \emph{threshold-based surrogates} (AT, IT and CL), which rank among the most popular ordinal regression loss functions and for which its Fisher consistency has not been studied previously.

Fisher consistency has also been studied in the pairwise ranking setting, where it has been proven~\citep{Duchi2010,calauzenes2012non} that some models (such as RankSVM) are not consistent. Despite similarities between ranking and ordinal regression, we will see in this paper that most popular ordinal regression models are Fisher consistent under mild conditions.

There are few studies on the theoretical properties of ordinal regression methods. A notable example comes from~\citet{Agarwal2008}, where the authors study generalization bounds for some ordinal regression algorithms. Some of the surrogate loss functions used by these models (such as the support vector ordinal regression of~\citet{Keerthi2003}) are analyzed in this paper. In that work, the authors outline the study of consistency properties of ordinal regression models as an important question to be addressed in the future.

A related, yet different, notion of consistency is \emph{asymptotic consistency}. A surrogate loss is said to be asymptotically consistent if the minimization of the $\psi$-risk converges to the optimal risk as the number of samples tends to infinity. It has also been studied in the setting of supervised learning~\citep{stone1977consistent,Steinwart2002}. This paper focuses solely on Fisher consistency, to whom we will refer simply as consistency from now on.

\section{Ordinal regression models}\label{sec:problem_setting}

We introduce the different ordinal regression models that we will consider within this paper. Considering first the absolute error, we will write this loss as a sum of binary 0-1 loss functions\footnote{The 0-1 loss, defined as the function that is $1$ for negative values and $0$ otherwise can be defined in bracket notation as $\ell_{0-1}(t) = \llbracket \alpha_i \leq 0 §\rrbracket$.}. This is a key reformulation of the absolute error that we will use throughout the paper. For any $y \in \mathcal{Y}$ and $\alpha \in \mathcal{S}$ we have the following sequence of equivalences
\begin{equation}\label{eq:development_absolute_error}
\begin{aligned}
\ell(y, \alpha) &= \abs{y - \text{pred}(\alpha)} = \abs{y - 1 - \sum_{i=1}^{k-1} \llbracket  \alpha_i < 0 §\rrbracket } \\
&= \abs{y - 1 - \sum_{i=1}^{y-1} \llbracket  \alpha_i < 0 §\rrbracket - \sum_{i=y}^{k-1}\llbracket  \alpha_i < 0 §\rrbracket} \\
&=\abs{\sum_{i=1}^{y-1}\llbracket  \alpha_i \geq 0 \rrbracket - \sum_{i=y}^{k-1}\llbracket  \alpha_i < 0 §\rrbracket} \quad .
\end{aligned}
\end{equation}

If $\alpha_y \geq 0$ then the second summand of the last equation equals zero. Otherwise, if $\alpha_y < 0$, then the first summand equals zero. In either case, we have
\begin{equation}\label{eq:absolute_value}
\ell(y, \alpha) = \sum_{i=1}^{y-1}\llbracket  \alpha_i \geq 0 \rrbracket + \sum_{i=y}^{k-1}\llbracket  \alpha_i < 0 §\rrbracket \quad .
\end{equation}

This expression suggests that a natural surrogate can be constructed by replacing the binary 0-1 loss in the above expression function by a convex surrogate such as the logistic or hinge loss. Denoting by $\varphi: \RR \to \RR$ such surrogate, we obtain the following loss function that we denote \emph{all threshold (AT)}:
\begin{equation}\label{eq:def_all_thresh}
\psi_\text{AT}(y, \alpha) := \sum_{i=1}^{y-1} \varphi(-\alpha_i) + \sum_{i=y}^{k-1} \varphi(\alpha_i) \quad.
\end{equation}
This function has appeared under different names in the literature. When $\varphi$ is the hinge loss, this model is known as support vector ordinal regression with implicit constraints~\citep{Keerthi2003} and support vector with sum-of-margins strategy~\citep{Shashua}. When $\varphi$ is the exponential loss, this model has been described in~\citep{lin2006large} as ordinal regression boosting with all margins. Finally,~\citet{Rennie} provided a unifying formulation for this approach considering for the hinge, logistic and exponential loss under the name of All-Threshold loss, a name that we will adopt in this paper.

The name \emph{thresholds} comes from the fact that in the aforementioned work, the decision function is of the form $\alpha_i = \theta_i - f(\cdot)$, where $(\theta_1, \ldots, \theta_{k-1})$ is a vector estimated from the data known as the vector of thresholds. We will discuss in Section~\ref{scs:parametric_consistency} the implications of such decision function. For the prediction rule to give meaningful results it is important to ensure that the thresholds are ordered, i.e., $\theta_1 \leq \theta_2, \leq \cdots, \leq \theta_{k-1}$~\citep{Keerthi2003}. In our setting, we enforce this through the constraint $\alpha \in \mathcal{S}$, hence the importance of restricting the problem to this subset of $\RR^{k-1}$.

Another family of surrogate loss functions takes a probabilistic approach and models instead the posterior probability. This is the case of the \emph{cumulative link} models of~\citet{McCullagh1980}. In such models the decision function $f$ is selected to approximate $\sigma(f_i(x)) = P(Y \leq i|X\!=\!x)$, where $\sigma: \RR \to [0, 1]$ is a function referred to as \emph{link function}. Several functions can be used as link function, although the most common ones are the sigmoid function and the Gaussian cumulative distribution. The sigmoid function, i.e., $\sigma(t) = 1/(1 + \exp(-t))$, leads to a model sometimes referred as {proportional odds}~\citep{McCullagh1980} and {cumulative logit}~\citep{agresti2010analysis}, although for naming consistency we will refer to it as \emph{logistic cumulative link}. Another important link function is given by the Gaussian cumulative distribution, $\sigma(t) = \frac{1}{\sqrt{2\pi}} \int_{-\infty}^t e^{-x^2/2}$, used in the Gaussian process ordinal regression model of~\citet{Chu2005a}. The cumulative link (CL) loss function is given by its negative likelihood, that is,
\begin{equation}\label{eq:propodds}
\psi_\text{CL}(y, \alpha) :=
\begin{cases}
-\log(\sigma(\alpha_1)) &\text{ if } y = 1 \\
-\log(\sigma(\alpha_y) - \sigma(\alpha_{y-1})) &\text{ if } 1 < y < k \\
-\log(1 - \sigma(\alpha_{k-1})) &\text{ if } y = k \quad.
\end{cases}
\end{equation}

We will now consider the multiclass 0-1 loss. In this case, the loss will be 1 if the prediction is below or above $y$ (i.e., if $\alpha_{y-1} \geq 0$ or $\alpha_y < 0$) and 0 otherwise. Hence, it is also possible to write the multiclass 0-1 loss as a sum of binary 0-1 loss functions:
\begin{equation*}
\ell(y, \alpha) = \begin{cases}
\llbracket \alpha_{1} < 0 \rrbracket &\text{ if } y = 1\\
\llbracket \alpha_{y-1} \geq 0 \rrbracket + \llbracket \alpha_{y} < 0 \rrbracket &\text{ if } 1 < y < k \\
\llbracket \alpha_{k-1} \geq 0 \rrbracket &\text{ if } y = k \quad.\\
\end{cases}
\end{equation*}
Given this expression, a natural surrogate is given by replacing the binary 0-1 loss by a convex surrogate as the hinge or logistic function. Following~\citet{Rennie}, we will refer to this loss function as \emph{immediate threshold (IT)}:
\begin{equation}\label{eq:immediate_threshold}
\psi_\text{IT}(y, \alpha) :=
\begin{cases}
\varphi(\alpha_{1}) &\text{ if } y = 1\\
\varphi(-\alpha_{y-1}) + \varphi(\alpha_y)&\text{ if } 1 < y < k \\
\varphi(-\alpha_{k-1}) &\text{ if } y = k \quad.
\end{cases}
\end{equation}
As with the AT surrogate, this loss has appeared under a variety of names in the literature. When $\varphi$ is the hinge loss, this model is known as support vector ordinal regression with explicit constraints~\citep{Keerthi2003} and support vector with fixed-margins strategy~\citep{Shashua}. When $\varphi$ is the exponential loss, this model has been described by~\citet{lin2006large} as ordinal regression boosting with left-right margins. We note that the construction of the AT and IT surrogates are similar, and in fact, we will see in Section~\ref{sct:extension_other_loss} that both can be seen as a particular instance of a more general family of loss functions.

\hfill

The aforementioned approaches can be seen as methods that adapt known binary classification methods to the ordinal regression setting. A different approach consists in treating the labels as real values and use regression algorithms to learn a real-valued mapping between the samples and the labels. This ignores the discrete nature of the labels, thus it is necessary to introduce a prediction function that converts this real value into a label in $\mathcal{Y}$. This prediction function is given by rounding to the closest label (see, e.g.,~\citep{kramer2001prediction} for a discussion of this method using regression trees). This approach is commonly referred to as the \emph{regression-based} approach to ordinal regression. If we are seeking to minimize the absolute error, a popular loss function is to minimize the least absolute deviation (LAD). For any $\beta \in \RR$, this is defined as
$$
\psi_\text{LAD}(y, \beta) := \abs{y - \beta} \quad,
$$
and prediction is then given by rounding $\beta$ to the closest label. This setting departs from the approaches introduced earlier by using a different prediction function. However, via a simple transformation it is possible to convert this prediction function (rounding to the closest label) to the prediction function that counts the number of non-zero components defined in Eq.~\eqref{eq:pred}. For a given $\beta \in \RR$, this transformation is given by
\begin{equation}\label{eq:lad_transform}
\alpha_1 = \frac{3}{2} - \beta, \quad \alpha_2 = \frac{5}{2} - \beta ,\quad \ldots ,\quad \alpha_{i} = i + \frac{1}{2} - \beta \quad.
\end{equation}
It is immediate to see that this vector $\alpha$ belongs to $\mathcal{S}$ and
$$
\begin{aligned}
\text{pred}(\alpha) &= 1 + \sum_{i=1}^{k-1} \llbracket i+\frac{1}{2} < \beta \rrbracket \\
&= \begin{cases}
1 \quad\text{ if } \beta \leq 1 + \frac{1}{2} \\
i \quad\text{ if } i - \frac{1}{2} \leq \beta < i + \frac{1}{2}, 1 < i < k \\
k \quad\text{ if } \beta \geq k - \frac{1}{2}
\end{cases} \\
&= \argmin_{1\leq i \leq k}\abs{\beta -i} \quad \text{(rounding to the lower label in case of ties)} \quad,
\end{aligned}
$$
hence predicting in the transformed vector $\alpha$ is equivalent to the closest label to $\beta$. We will adopt this transformation when considering LAD for convenience, in order to analyze it within the same framework as the rest. With the aforementioned transformation, the least absolute deviation surrogate is given by
\begin{equation}
\label{eq:LAD_definition}
\psi_\text{LAD}(y, \alpha) = \abs{y + \alpha_1 - \frac{3}{2}}
\end{equation}
Although the surrogate loss function LAD and the absolute loss of Eq.~\eqref{eq:absolute_loss} look very similar, they differ in that the LAD surrogate is convex on $\alpha$, while the absolute error is not, due to the presence of the discontinuous function $\text{pred}(\cdot)$.

\hfill

In this section we have introduced some of the most common ordinal regression methods  based on the optimization of a convex loss function. These are summarized in Table~\ref{sample-table}.

\begin{table}[ht]
\begin{mdframed}
\caption{Surrogate loss functions considered in this paper.} \label{sample-table}
\begin{center}
\begin{tabular}{p{4.0cm} c p{4.8cm}}
{\bfseries Model}  &{\bfseries Loss Function}  &{\bfseries Also known as} \\
\hline
All thresholds (AT) & $\sum_{i=1}^{y-1} \varphi(-\alpha_i) + \sum_{i=y}^{k-1}\varphi(\alpha_i)$ & Implicit constraints~\citep{Keerthi2003}, all margins~\citep{lin2006large}. \\
Cumulative link (CL) & $-\log(\sigma(\alpha_y) - \sigma(\alpha_{y-1}))$  & Proportional odds~\citep{McCullagh1980}, cumulative logit~\citep{agresti2010analysis}. \\
Immediate threshold (IT) & $\varphi(-\alpha_{y-1}) + \varphi(\alpha_{y})$ & Explicit constraints~\citep{Keerthi2003}, Fixed-margins~\citep{Shashua} \\
Least absolute deviation (LAD) & $| y + \alpha_1 - \frac{3}{2}|$ &  Least absolute error, least absolute residual, Sum of absolute deviations, $\ell_1$ regression. \\
Least squares (LS) & $\left( y + \alpha_1 - \frac{3}{2}\right)^2$ &  Squared error, sum of squares, $\ell_2$ regression. \\
\end{tabular}
\end{center}
\end{mdframed}
\end{table}

\section{Consistency results}

In this section we present consistency results for different surrogate loss functions. We have organized this section by the different loss functions against which we test for consistency. The first subsection presents results for the absolute error, which is the most popular loss for ordinal regression. In the second subsection we provide consistency results for a surrogate of the squared loss. Finally, in the third subsection we show results for the 0-1 loss as, perhaps surprisingly, several commonly used surrogates turn out to be consistent with respect to this loss.

\subsection{Absolute error surrogates}\label{sct:absolute_error_surrogates}

In this section we will assume that the loss function is the absolute error, i.e., $\ell(y, \alpha) = \abs{y - \text{pred}(\alpha)}$ and we will focus on surrogates of this loss.
For an arbitrary $\alpha \in \mathcal{S}$,  the conditional risk for the absolute error can be reformulated using the development of the absolute error from Eq.~\eqref{eq:development_absolute_error}:
\begin{equation*}
\begin{aligned}
L(\alpha, p) &=  \sum_{i=1}^k p_i \left(\sum_{j=1}^{i-1}\llbracket  \alpha_j \geq 0 \rrbracket + \sum_{j=i}^{k-1}\llbracket  \alpha_j < 0 §\rrbracket \right) \\
 &=  \sum_{i=1}^k \llbracket  \alpha_j \geq 0 \rrbracket (1 - u_i(p)) + \sum_{j=1}^{k}  \llbracket  \alpha_j < 0 \rrbracket u_i(p)  \quad,\\
\end{aligned}
\end{equation*}
where $u(p)$ is the vector of cumulative probabilities, i.e., $u_i(p) := \sum_{j=1}^i p_j$. Let $r = \text{pred}(\alpha)$. Then $\alpha_{r-1} < 0$ and $\alpha_r \geq 0$, from where the above can be simplified to
\begin{equation}\label{eq:conditional_risk}
\begin{aligned}
L(\alpha, p) &=  \sum_{i=1}^{r-1} u_i(p) + \sum_{i=r}^{k-1} (1 - u_i(p))\quad. \\
\end{aligned}
\end{equation}
Using this expression, we will now derive an explicit minimizer of the conditional risk. Note that because of the prediction function counts the number of nonzero coefficients,  only the sign of this vector is of true interest.
\begin{lemma}\label{lemma:bayes_decision_function} For any $p \in \Delta^k$, let $\underline{\alpha}(p)$ be defined as
$$\underline{\alpha}(p) = (2 u_1(p) - 1, \ldots, 2 u_{k-1}(p) - 1)\quad.$$
Then, $L(\cdot, p)$ achieves its minimum at $\underline{\alpha}(p)$, that is,
 $$\underline{\alpha}(p) \in \argmin L(\alpha, p)\quad.$$
\end{lemma}
\begin{proof}
We will prove that for any $\alpha \in \mathcal{S}$ and any $p \in \Delta^k$, $L(\alpha, p) \geq L(\underline{\alpha}(p), p)$. We consider $p$ and $\alpha$ fixed and we denote $r^* = \text{pred}(\underline{\alpha}(p))$ and $r = \text{pred}(\alpha)$. We distinguish three cases, $r < r^*$, $r > r^*$ and $r = r^*$.

\begin{itemize}
\item {$r < r^*$}. In this case, Eq.~\eqref{eq:conditional_risk} implies that
\begin{equation*}
L(\alpha, p) - L(\underline{\alpha}(p), p) = - \sum_{i=r}^{r^*-1} u_i(p) +  \sum_{i=r}^{r^*-1} (1 - u_i(p)) = -\sum_{i=r}^{r^* - 1} 2 u_i(p) - 1 \quad.
\end{equation*}
Now, by the definition of prediction function, $2 u_i(p) - 1 < 0$ for $i < r^*$, so we have
$$
L(\alpha, p) - L(\underline{\alpha}(p), p) = \sum_{i=r}^{r^* - 1} \abs{2 u_i(p) - 1} \quad.
$$

\item $r > r^*$. Similarly, in this case Eq.~\eqref{eq:conditional_risk} implies that
\begin{equation*}
L(\alpha, p) - L(\underline{\alpha}(p), p) = \sum_{i=r^*}^{r-1} u_i(p) -  \sum_{i=r^*}^{r-1} (1 - u_i(p)) = \sum_{i=r^*}^{r - 1} 2 u_i(p) - 1\quad .
\end{equation*}
Since  by definition of prediction function $2 u_i(p) - 1 \geq 0$ for $i \geq r^*$, it is verified that
$$
L(\alpha, p) - L(\underline{\alpha}(p), p) = \sum_{i=r^*}^{r - 1} \abs{2 u_i(p) - 1}\quad .
$$
\item $r = r^*$. In this case, Eq.~\eqref{eq:conditional_risk} yields
$$
L(\alpha, p) - L(\underline{\alpha}(p), p) = 0 \quad.
$$
\end{itemize}

Let $I$ denote the set of indices for which $\alpha$ disagrees in sign with $\underline{\alpha}$, that is, $I = {\{ i: \alpha_i (2 u_i(p) - 1) < 0\}}$. Then, combining the three cases we have the following formula for the excess in conditional risk
\begin{equation}\label{eq:excess_risk}
  L(\alpha, p) - L(\underline{\alpha}(p), p) = \sum_{i \in I} \abs{2 u_i(p) - 1} \quad ,
\end{equation}
which is always non-negative and hence $L^*(p) = L(\underline{\alpha}(p), p)$.
\end{proof}

{\bfseries All threshold (AT)}. We will now consider the AT surrogate. We will prove that some properties known for binary classification are inherited by this loss function. More precisely, we will provide a characterization of consistency for convex $\varphi$ in Theorem~\ref{thm:all_threshold} and excess risk bounds in Theorem~\ref{thm:excess_risk} that parallel those of~\citet{Bartlett2003} for binary classification.

Through this section $A$ will represent the conditional risk of the AT surrogate, which can be expressed as:
\begin{equation}\label{eq:all_thresh_risk}
\begin{aligned}
A(\alpha, p) &= \sum_{j=1}^k p_j \psi_{\text{AT}}(j, \alpha) = \sum_{j=1}^k p_j \left( \sum_{i=1}^{j-1} \varphi(-\alpha_i) + \sum_{i=j}^{k-1} \varphi(\alpha_i) \right) \\
&= \sum_{i=1}^{k-1} (1 - u_i(p)) \varphi(-\alpha_i)  + u_i(p) \varphi(\alpha_i) \quad,
\end{aligned}
\end{equation}
where as in the previous section $u_i(p) = \sum_{j=1}^i p_i, \alpha \in \mathcal{S}$ and $p \in \Delta^k$. This surrogate verifies a decomposable property that will be key to further analysis. The property that we are referring to is that the above conditional risk it can be expressed as the sum of $k-1$ binary classification conditional risks. For $\beta \in \RR, q \in [0, 1]$, we define $C$ as follows
$$
C(\beta, q) = q \varphi(\beta) + (1 - q) \varphi(-\beta) \quad,
$$
where $C$ can be seen as the conditional risk associated with the binary classification loss function $\varphi$. Using this notation, the conditional risk $A$ can be expressed in terms of $C$ as:
$$
A(\alpha, p) = \sum_{i=1}^{k-1} C(\alpha_i, u_i(p)) \quad.
$$

Our aim is to compute $A^*$ in terms of the infimum of $C$, denoted $C^*(q) := \inf_{\beta} C(q, \beta)$. Since $C$ is the conditional risk of a binary classification problem, this would yield a link between the optimal risk for the AT surrogate and the optimal risk for a binary classification surrogate. However, this is in general not possible because of the monotonicity constraints in $\mathcal{S}$: the infimum over $\mathcal{S}$ need not equal the infimum over the superset $\RR^{k-1}$. We will now present a result that states sufficient conditions under which the infimum over $\mathcal{S}$ and over $\RR^{k-1}$ do coincide. This implies that $A^*$ can be estimated as the sum of $k-1$ different surrogate conditional risks, each one corresponding to a binary classification surrogate. A similar result was proven in the empirical approximation setting by \citet[Lemma 1]{Keerthi2003}. In this work, the authors consider the hinge loss and show that any minimizer of this loss automatically verifies the monotonicity constraints in $\mathcal{S}$.

In the following lemma we give sufficient conditions on $\varphi$ under which the monotonicity constraints can be ignored when computing $A^*$. This is an important step towards obtaining an explicit expression for $A^*$:
\begin{lemma}\label{lemma:factor_Ai} Let $\varphi: \RR \to \RR$ be a function such that
$
\varphi(\beta) - \varphi(-\beta)
$ is a non-increasing function of $\beta \in \RR$.
Then for all $p \in \Delta^k$, it is verified that
$$
A^*(p) = \sum_{i=1}^{k-1} C^*(u_i(p)) \quad.
$$
\end{lemma}

\begin{proof} Let $p \in \Delta^k$ be fixed and let $\alpha^* \in \argmin_{\alpha \in \RR^{k-1}} A(\alpha, p)$. If $\alpha^* \in \mathcal{S}$, then the result is immediate since
$$
\sum_{i=1}^{k-1} C^*(u_i(p)) = A(\alpha^*, p) = \inf_{\alpha \in \mathcal{S}} A(\alpha, p) = A^*(p) \enspace .
$$
Suppose now $\alpha^* \notin \mathcal{S}$. We will prove that in this case it is possible to find another vector $\tilde{\alpha} \in \mathcal{S}$ sith the same surrogate risk. By assumption there exists a $i$ in the range $1 \leq i \leq k-2$ for which the monotonicity conditions in $\mathcal{S}$ are not verified. In this case it is verified that $\alpha_{i+1} < \alpha_i$. Since $(u_1(p),\ldots, u_{k-1}(p))$ is a non-decreasing sequence, for a fixed $p$ it is possible to write $u_{i+1}(p) = u_i(p) + \varepsilon$, with $\varepsilon \geq 0$. Then it is true that
$$
\begin{aligned}
C(\alpha^*_i, u_{i+1}(p)) &= (1 - u_i(p) - \varepsilon)\varphi(-\alpha^*_i) + (u_i(p) + \varepsilon) \varphi(\alpha^*_i) \\
&=C(\alpha^*_i, u_i(p)) + \varepsilon(\varphi(\alpha^*_i) - \varphi(-\alpha^*_i)) \quad.
\end{aligned}
$$
By assumption $\varepsilon(\varphi(\alpha^*_i) - \varphi(-\alpha^*_i))$ is a non-increasing function of $\alpha^*_i$ and so $\alpha^*_{i+1} < \alpha^*_i \implies C(\alpha_i, u_{i+1}(p)) \leq C(\alpha_{i+1}, u_{i+1}(p))$. By the optimality of $\alpha_{i+1}^*$, it must be $C(\alpha^*_i, u_{i+1}(p)) = C(\alpha^*_{i+1}, u_{i+1}(p))$. This implies that the vector in which $\alpha^*_{i+1}$ is replaced by $\alpha_i^*$ has the same conditional risk and hence suggest a procedure to construct a vector that satisfies the constraints in $\mathcal{S}$ and achieves the minimal risk in $\RR^{k-1}$. More formally, we define $\tilde{\alpha} \in \mathcal{S}$ as:$$
\tilde{\alpha}_i = \begin{cases}
\alpha^*_1 \text{ if } i=1 \\
\alpha^*_i \text{ if } \alpha^*_{i-1} \leq \alpha^*_{i} \\
\alpha^*_{i-1} \text{ if } \alpha^*_{i-1} > \alpha^*_{i} \quad.
\end{cases}
$$
Then by the above $C(\alpha^*_i, u_i(p)) = C(\tilde{\alpha}_i, u_i(p))$ for all $i$ and so
$
A(\alpha^*, p) = A(\tilde{\alpha}, p)
$. Now, since $\tilde{\alpha}$ is a non-decreasing vector by construction, $\tilde{\alpha} \in \mathcal{S}$ and we have the sequence of equalities
$$
A(\alpha^*, p) = A(\tilde{\alpha}, p) = \sum_{i=1}^{k-1} C^*(u_i(p)) \quad,
$$
which completes the proof.

\end{proof}

It is easy to verify that the condition on $\varphi$ of this theorem is satisfied by all the binary losses that we consider: hinge loss, the squared hinge loss, the logistic loss, exponential loss and the squared loss. With this result, if $\alpha^*_i$ is a minimizer of $C(u_i(p))$, then $(\alpha_1^*, \ldots, \alpha^*_{k-1})$ will be a minimizer of $A(p)$. Hence, the optimal decision function for the aforementioned values of $\varphi$ is simply the concatenation of known results for binary classification, which have been derived in~\citet{Bartlett2003} for the hinge, squared hinge and Exponential loss and in~\citep{Zhang} for the logistic loss. Using the results from binary classification we list the values of $\alpha^*$ and $A^*$ in the case of AT for different values of $\varphi$:
\begin{itemize}
  \item \emph{Hinge AT}, : $\alpha^*_i(p) = \text{sign}(2 u_i(p) - 1)$, \quad $A^*(p) = \sum_{i=1}^{k-1}\{ 1 - \abs{2 u_i(p) - 1}\}$.
  \item \emph{Squared hinge AT}, : $\alpha^*_i(p) = (2 u_i(p) - 1)$, \quad $A^*(p) = \sum_{i=1}^{k-1} 4  u_i(p) (1 - u_i(p))$.
  \item \emph{Logistic AT}: $\alpha^*_i(p) = \log\left(\frac{u_i(p)}{1 - u_i(p)}\right)$, $A^*(p)$ = ${\sum_{i=1}^{k-1} \{-E(u_i(p)) - E(1 - u_i(p))\}}$, where $E(t) = t \log(t)$.
  \item \emph{Exponential AT}: $\alpha^*_i(p) = {\frac{1}{2} \log\left(\frac{u_i(p)}{1 - u_i(p)}\right)}$, \quad $A^*(p) = {\sum_{i=1}^{k-1} 2 \sqrt{u_i(p) (1 - u_i(p))}}$.
  \item \emph{Squared AT}, $\alpha_i^*(p) = 2 u_i(p) - 1$, \quad $A^*(p) = \sum_{i=1}^{k-1} (2 - 2 u_i(p))^2$.
\end{itemize}

It is immediate to check that the models mentioned above are consistent since the decision functions coincides in sign with the minimizer of the risk defined in Lemma~\ref{lemma:bayes_decision_function}. Note that the sign of $\alpha^*_i(p)$ at $u_i(p) = \frac{1}{2}$ is irrelevant, since by Eq.~\eqref{eq:conditional_risk} both signs have equal risk. We now provide a result that characterizes consistency for a convex $\varphi$:

\begin{theorem}\label{thm:all_threshold}
  Let $\varphi: \RR \to \RR$ be convex. Then the AT surrogate is consistent if and only if $\varphi$ is differentiable at $0$ and $\varphi'(0) < 0$.
\end{theorem}
\begin{proof} We postpone the proof until Section~\ref{sct:extension_other_loss}, where this will follow as a particular case of Theorem~\ref{thm:main_thm}.
\end{proof}

We will now derive excess risk bounds for AT. These are inequalities that relate the excess conditional risk $L(\alpha) - L^*$, to the excess in surrogate conditional risk $A(\alpha) - A^*$. For this, we will make use of the $\gamma$-transform\footnote{\citet{Bartlett2003} define this as the $\psi$-transform. However, since we already use $\psi$ to denote the surrogate loss functions we will use letter $\gamma$ in this case.} of a binary loss function~\citep{Bartlett2003}. For a convex function $\varphi$ this is defined as
\begin{equation}\label{eq:psi_transform}
\gamma(\theta) = \varphi(0) - C^*\left(\frac{1 + \theta}{2}\right) \quad.
\end{equation}
We will now state the excess risk bound of the AT surrogate in terms of the $\gamma$-transform:

\begin{theorem}[Excess risk bounds]\label{thm:excess_risk}
Let $\varphi: \RR \to \RR$ be a function that verifies the following conditions:
\begin{itemize}
\item $\varphi$ is convex.
\item $\varphi$ is differentiable at $0$ and $\varphi'(0) < 0$.
\item $\varphi(\beta) - \varphi(-\beta)$ is a non-increasing function of $\beta$.
\end{itemize}
Then for any $\alpha \in \mathcal{S}, p \in \Delta^k$, the following excess risk bound is verified:
\begin{equation}\label{eq:excess_risk_bounds}
\gamma\left(\frac{L(\alpha, p) - L^*(p)}{k-1}\right) \leq \frac{A(\alpha, p) - A^*(p)}{k-1} \quad.
\end{equation}
\end{theorem}
\begin{proof}
  Let ${I}$ denote the set of indices in which the sign of $\alpha$ does not coincide with $\underline{\alpha}$, that is, $I = \{ i: \alpha_i(2 u_i(p) - 1) < 0 \}$. From~\citet[Lemma 7]{Bartlett2003}, we know that if $\varphi$ is convex and consistent (in the context of binary classification), then $\gamma$ is convex and we can write
  \begin{equation}
  \begin{aligned}\label{eq:risk2}
    \gamma\left(\frac{L(\alpha, p) - L^*(p)}{k-1}\right)& = \gamma\left(\frac{\sum_{i \in I} |2 u_i(p) - 1|}{k-1}\right) &\text{ (by Eq.~\eqref{eq:excess_risk})}\\
    &\leq \frac{\sum_{i \in I} \gamma(|2 u_i(p) - 1|)}{k-1} &\text{ (by Jensen's inequality)} \\
    &= \sum_{i \in I} \frac{\gamma(2 u_i(p) - 1)}{k-1} &\text{ (by symmetry of $\gamma$)} \\
    &= \sum_{i \in I} \frac{\varphi(0) -  C^*\left(u_i(p)\right)}{k-1}&\text{ (by definition of $\gamma$)}.
  \end{aligned}
  \end{equation}
Let $q \in [0, 1], \beta \in \RR$. If we can further show that $\beta ( 2 q - 1) \leq 0$ implies $\varphi(0) \leq C(\beta, q)$, then
$$
\begin{aligned}
\sum_{i \in I}\left(\varphi(0) - C^*(u_i(p))\right) &\leq \sum_{i \in I}\left(C(\alpha_i, u_i(p)) - C^*(u_i(p))\right) \\
&\leq  \sum_{i=1}^{k-1} \left(C(\alpha_i, u_i(p)) - C^*(u_i(p))\right) \\
&= A(\alpha, p) - A^*(p) \quad \text{(by Lemma~\ref{lemma:factor_Ai})}.
\end{aligned}
$$
Combining this inequality with Eq.~\eqref{eq:risk2}, we obtain the theorem. Therefore we only need to prove that $\beta ( 2 q - 1) \leq 0$ implies $\varphi(0) \leq C(\beta, q)$. Suppose $\beta ( 2 q - 1) \leq 0$. Then by Jensen's inequality
$$
C(\beta, q) = q \varphi(\beta) + (1 - q) \varphi(-\beta) \geq \varphi(q \beta - (1 - q) \beta) = \varphi(\beta(2 q - 1)) \quad.
$$
Now, by convexity of $\varphi$ we have
$$
\varphi(\beta(2 q - 1)) \geq \varphi(0) + \beta(2 q - 1) \varphi'(0)\geq\varphi(0) \quad,
$$
where the last inequality follows from the fact that $\varphi'(0) < 0$ and $\beta(2 q - 1) \leq 0$. This concludes the proof.
\end{proof}

Note that we have given the excess risk bounds in terms of the conditional risk. These can also be expressed in terms of the full risk, as done for example by~\citet{Bartlett2003,Zhang}. Within the conditions of the theorem, $\gamma$ is convex and because of Jensen inequality, it is verified that
$$
\gamma\left(\EE_X\left[L(f(X), \eta(X)) - L^*(\eta(X))\right]\right) \leq \EE_X\left[\gamma(L(f(X), \eta(X))\right] \quad.
$$
This, together with Eq.~\eqref{eq:excess_risk_bounds} yields the following bound in terms of the full risk
$$
\begin{aligned}
\gamma\left(\frac{\mathcal{L}(f) - \mathcal{L}}{k-1}\right) &\leq \EE_X \left[ \gamma\left(\frac{L(f(X), \eta(X)) - L^*(\eta(X))}{k-1}\right)\right] \\
&\leq \EE_X \left[ \frac{A(f(X), \eta(X)) - A^*(\eta(X))}{k-1}\right] \\
&= \frac{\mathcal{A}(f) - \mathcal{A}^*}{k-1}
\end{aligned}
$$

{\bfseries Examples of excess risk bounds}. We will now derive excess bounds for different instances of the AT loss function. The values of $\gamma$ only depend on $\varphi$, so we refer the reader to~\citet{Bartlett2003} on the estimation of $\gamma$ for the hinge, squared hinge and Exponential loss and to~\citep{Zhang} for the logistic loss. Here, we will merely apply the known form of $\gamma$ to the aforementioned surrogates.

\begin{itemize}
  \item \emph{Hinge AT}, : $\gamma(\theta) = \abs{\theta} \implies L(\alpha, p) - L^*(p) \leq A(\alpha, p) - A^*$.
  \item \emph{Squared hinge AT}, : $\gamma(\theta) = \theta^2 \implies \left(\frac{L(\alpha, p) - L^*(p)}{k-1}\right)^2 \leq A(\alpha, p) - A^*$.
  \item \emph{Logistic AT}: $\gamma(\theta) = \frac{\theta^2}{2} \implies \left(\frac{L(\alpha, p) - L^*(p)}{\sqrt{2} (k-1)}\right)^2 \leq A(\alpha, p) - A^*$.
  \item \emph{Exponential AT}: $\gamma(\theta) = 1 - \sqrt{1 - \theta^2} \implies$ \\ $(k-1)(1 - \sqrt{1 - \frac{(L(\alpha, p) - L^*(p))^2}{k-1}}) \leq A(\alpha, p) - A^*$.
  \item \emph{Squared AT}: $\gamma(\theta) = \theta^2 \implies \left(\frac{L(\alpha, p) - L^*(p)}{k-1}\right)^2 \leq A(\alpha, p) - A^*$ .
\end{itemize}

For $k=2$, these results generalize the known excess risk bounds for binary surrogates. For $k>2$, the normalizing factor $\frac{1}{k-1}$ is not surprising, since the absolute error is bounded by $k-1$ while the 0-1 loss is bounded by $1$. While similar excess risk bounds are known for multiclass classification~\citep{Zhang2004,avila2013cost}, to the best of our knowledge this is the first time that such bounds have been developed for the AT surrogate ($k > 2$).

{\bfseries Cumulative link (CL)}. We now focus on the CL loss function defined in Eq.~\eqref{eq:propodds}, which we restate here for convenience:
\begin{equation}\label{eq:propodds}
\psi_\text{CL}(y, \alpha) :=
\begin{cases}
-\log(\sigma(\alpha_1)) &\text{ if } y = 1 \\
-\log(\sigma(\alpha_y) - \sigma(\alpha_{y-1})) &\text{ if } 1 < y < k \\
-\log(1 - \sigma(\alpha_{k-1})) &\text{ if } y = k \quad.
\end{cases}
\end{equation}
The terms of this surrogate can be seen as the negative log-likelihood of a probabilistic model in which $\sigma(\alpha)$ model the cumulative probabilities:
\begin{equation}\label{eq:cl_optimal}
\sigma(\alpha_i) = P(Y \leq i | X = x) = u_i(p)\quad,
\end{equation}
and so the likelihood is maximized for $\sigma(\alpha_i^*) = u_i(p)$. Assuming that the inverse of the link function $\sigma$ exists, this implies that the minimizer of the surrogate loss function (given by the negative log-likelihood) is given by $\alpha^*_i(p) = \sigma^{-1}(u_i(p))$. Plugging into the formula for the surrogate risk yields $A^*(p) = \sum_{i=1}^k p_i \log(p_i)$.
This immediately leads to a characterization of consistency based on the link function $\sigma$:

\begin{theorem}\label{thm:consistency_cl}
Suppose $\sigma$ is an invertible function. Then the CL surrogate is consistent if and only if the inverse link function verifies
\begin{equation}\label{eq:condition_cl}
\begin{aligned}
  \left(\sigma^{-1}(t)\right)(2 t - 1) > 0 &\text{ for } t \neq \frac{1}{2} \\
  \end{aligned}
  \quad.
\end{equation}
\end{theorem}
\begin{proof}
  ($\implies$) Suppose CL is consistent but $\sigma^{-1}$ does not verify Eq.~\eqref{eq:condition_cl}, i.e., there exists a $\xi \neq 1/2$ such that $\sigma^{-1}(\xi)(2 \xi - 1) \leq 0$. We consider a probability distribution $P$ such that $u_1(p) = \xi$ for all $p \in \Delta^k$. In that case, by Eq.~\eqref{eq:cl_optimal} $\alpha^*_1(p) = \sigma^{-1}(\xi)$ and so this has a sign opposite to the Bayes decision function $2 \xi - 1$. By Eq.~\eqref{eq:excess_risk} this implies that $L(\alpha^*) - L^* \geq 2 \xi - 1 > 0$. The last inequality implies that $\alpha^*$ does not reach the minimal risk, contradiction since CL is consistent by assumption.

  $(\impliedby)$ Let $0<i<k$. For $u_i(p) \neq 1/2$, $\alpha^*_i(p) = \sigma^{-1}(u_i(p))$ by Eq.~\eqref{eq:cl_optimal} agrees in sign with $2 u_i(p) - 1$ and so by Lemma~\ref{lemma:bayes_decision_function} has minimal risk. If $u_i(p) = 1/2$, then in light of Eq.~\eqref{eq:conditional_risk} the risk is the same no matter the value of $\alpha^*_i(p)$. We have proven that $\alpha^*(p)$ has the same risk as a Bayes decision function, hence the CL model is consistent. This completes the proof.
\end{proof}
The previous theorem captures the notion that the inverse of the link function should agree in sign with $2 t - 1$.
When the link function is the sigmoid function, i.e., $\sigma(t) = 1 / (1 + e^{-t})$ this surrogate is convex and its inverse link function is given by the logit function, which verifies the assumptions of the theorem and hence is consistent. Its optimal decision function is given by
\begin{wrapfigure}{r}{0.45\textwidth}
\centering  \includegraphics[width=0.40\textwidth]{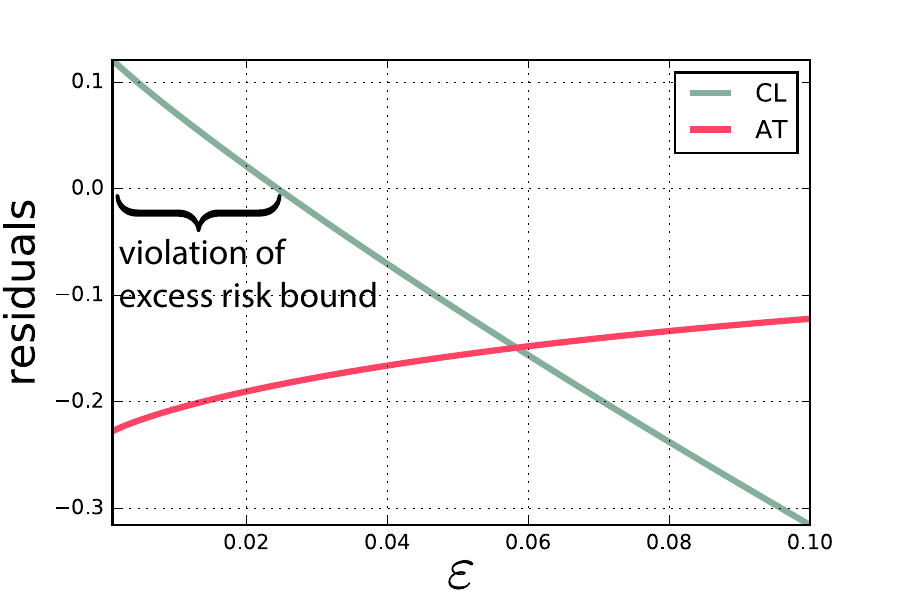}
  \caption{counterexample for the CL excess risk bound.}\label{fig:counterexample}
  \vspace{-20pt}
\end{wrapfigure}
$$\alpha_i^*(p) = \log \left( \frac{u_i(p)}{1 - u_i(p)}\right)\quad,$$
which coincides with the logistic AT surrogate. Despite the similarities between both surrogates, we have not been able to derive excess risk bounds for this surrogate since the separability properties of the AT are not met in this case. Furthermore, it is possible to construct a counter example that the $\gamma$-transform for the Logistic AT ($\gamma(\theta) = \theta^2 / 2$) loss does not yield a valid risk bound in this case. To see this, let $k=3$ and $p, \alpha$ be as follows:
$$
p = (1 - 2 \varepsilon,  1.5 \varepsilon,  0.5 \varepsilon),\qquad \alpha = (-0.1, -0.05) \quad.
$$
For these values we can compute the excess risk and excess surrogate risk as
$$
\begin{aligned}
L(\alpha, p) - L^* &= \sum_{i=1}^2 (2 u_i(p) - 1) = 2 - 6 \varepsilon \\
\mathcal{A}(\alpha, p) - \mathcal{A}^* &= \sum_{i=1}^{3}p_i \psi_{\text{CL}}(i, \alpha) - \sum_{i=1}^{3}p_i \log(p_i)
\end{aligned}
$$
If the risk bound is satisfied, then the residuals, defined as
$$
\text{residuals} = \gamma\left(\frac{L(\alpha, p) - L^*}{k-1}\right) - \frac{\mathcal{A}(\alpha, p) - \mathcal{A}^*}{k-1}
$$
must be always negative. However, as it can be seen in Figure~\ref{fig:counterexample}, the residuals are increasing as $\varepsilon$ goes to 0. Also, in the region $\varepsilon < 0.02$ the residuals become positive and hence the inequality does not hold.

We finish our treatment of the CL surrogate by stating the convexity of the logistic CL loss, which despite being a fundamental property of the loss, has not been proven before to the best of our knowledge.

\begin{lemma}\label{thm:convex}
The logistic CL surrogate, $\alpha \mapsto \psi_{\text{CL}}(y, \alpha)$, is convex on $\mathcal{S}$ for every value of $y$.
\end{lemma}

\begin{proof}
$\psi_{\text{CL}}(1, \alpha)$ and $\psi_{\text{CL}}(k, \alpha)$ are convex because they are log-sum-exp functions. It is thus sufficient to prove that $\psi_{\text{CL}}(i, \cdot)$ is convex for $1 < i< k$. For convenience we will write this function as $f(a, b) = -\log\left( \frac{1}{1 + \exp{(a)}} -
\frac{1}{1 + \exp{(b)}}\right)$, where $a > b$ is the domain of definition.

By factorizing the fraction inside $f$ to a common denominator, $f$ can
equivalently be written as $- \log(\exp(a) - \exp(b)) + \log(1 + \exp(a)) +
\log(1 + \exp(b))$. The last two terms are convex because they can be written
as a log-sum-exp. The convexity of the first term, or equivalently the  \mbox{log-concavity} of the function $f(a, b) = {\exp(a) - \exp(b)}$ can be
settled by proving the positive-definiteness of the matrix $Q = \nabla f(a, b)\nabla f(a, b)^T - f(a, b)\nabla^2f(a, b)$ for all $(a, b)$ in the
domain $\{b > a\}$~\citep{boyd2004convex}. In our case,
\begin{equation*}
Q =
\begin{pmatrix}
\exp(a + b) & -\exp(a + b) \\
- \exp(a + b) & \exp(a + b)
\end{pmatrix} =
\exp(a + b)\begin{pmatrix}
1 & -1 \\
- 1 & 1
\end{pmatrix}
\enspace ,
\end{equation*}
which is a positive semidefinite matrix with eigenvalues $2 \exp(a + b)$ and $0$. This
proves that $Q$ is positive semidefinite and thus $\psi_{\text{CL}}(i, \cdot)$ is a convex function.
\end{proof}

{\bfseries Least absolute deviation}. We will now  prove consistency of the least absolute deviation (LAD) surrogate. Consistency of this surrogate was already proven for the case $k=3$ by~\citet{Ramaswamy2012}. For completeness, we provide here an alternative proof for an arbitrary number of classes.

\begin{theorem}\label{thm:consistency_lad}
  The least absolute deviation surrogate is consistent.
\end{theorem}

\begin{proof}
Recall that for $y \in \mathcal{Y}, \alpha \in \mathcal{S}$, the LAD surrogate is given by
$$
\psi_{\text{LAD}}(y, \alpha) = \abs{y + \alpha_1 - \frac{3}{2}} \quad.
$$
The pointwise surrogate risk is then given by
$$
A(\alpha, p) = \sum_{i=1}^k p_i \psi_{\text{LAD}}(y, \alpha) = \EE_{Y \sim p}\left[\,\abs{Y + \alpha_1 - \frac{3}{2}}\right] \quad,
$$
where $Y\!\sim\! p$ means that $Y$ is distributed according to a multinomial distribution with parameter $p \in \Delta^k$. By the optimality conditions of the median, a value that minimizes this conditional risk is given by
$$
\alpha^*_1(p) \in \text{Median}_{Y \sim p}\left(\frac{3}{2} - Y\right) \quad,
$$
where $\text{Med}$ is the median, that is, $\alpha^*_1(p)$ is any value that verifies
$$
P\left(\frac{3}{2} - Y \leq \alpha^*_1(p)\right) \geq \frac{1}{2} \text{ and } P\left(\frac{3}{2} - Y \geq \alpha^*_1(p)\right) \geq \frac{1}{2} \quad.
$$
We will now prove that LAD is consistent by showing that $L(\alpha^*(p), p) = L(\underline{\alpha}(p), p)$, where $\underline{\alpha}$ is the Bayes decision function described in Lemma~\ref{lemma:bayes_decision_function}. Let $r^* = \text{pred}(\underline{\alpha}(p))$ and $I$ denote the set $I = \{i: \alpha^*_i(p) (2 u_i(p) - 1) < 0 \}$. Suppose this set is non-empty and let $i \in I$. We distinguish the cases $\alpha^*_i(p) > 0$ and $\alpha^*_i(p) < 0$:

\begin{itemize}
\item $\alpha^*_i(p) < 0$. By Eq.~\eqref{eq:lad_transform}, $\alpha^*_i$ and $\alpha^*_1$ are related by $\alpha^*_i = i-1 + \alpha^*_1$. Then it is verified that
$$
\begin{aligned}
P\left(\frac{3}{2} - Y \geq \alpha^*_1(p)\right) &= P\left(\frac{3}{2} - Y \geq \alpha_i^* - i + 1\right) = P\left(\frac{1}{2} + i - \alpha^*_i \geq Y\right)  \\
&\geq P\left(\frac{1}{2} + i \geq Y\right) = u_i(p) \quad.
\end{aligned}
$$
By assumption, $\alpha^*_i(p) (2 u_i(p) - 1) < 0$, which implies $u_i(p) > 1/2$. Hence, by the above we have that $P\left(\frac{3}{2} - Y \geq \alpha^*_1(p)\right) > 1/2$. At the same time, by the definition of median, $P\left(\frac{3}{2} - Y \geq \alpha^*_1(p)\right) \leq 1/2$, contradiction.

\item $\alpha^*_i(p) > 0$. Using the same reasoning as before, it is verified that
$$
\begin{aligned}
P\left(\frac{3}{2} - Y \leq \alpha^*_1(p)\right) &= P\left(\frac{3}{2} - Y \leq \alpha_i^* - i + 1\right) = P\left(\frac{1}{2} + i - \alpha^*_i \leq Y\right)  \\
&\geq P\left(\frac{1}{2} + i \leq Y\right) = 1 - u_i(p) \quad.
\end{aligned}
$$
By assumption $u_i(p) < 1/2 \implies P\left(\frac{3}{2} - Y \leq \alpha^*_1(p)\right) > 1/2$. At the same time, by the definition of median, $P\left(\frac{3}{2} - Y \geq \alpha^*_1(p)\right) \leq 1/2$, contradiction.
\end{itemize}

Supposing $I$ not empty has lead to contradictions in both cases, hence $I = \emptyset$. By Eq.~\eqref{eq:excess_risk}, $L(\alpha^*(p), p) = L(\underline{\alpha}(p), p)$, which concludes the proof.
\end{proof}

\subsection{Squared error}\label{scs:squared_error}

We now consider the squared error, defined as
$$
\ell(y, \alpha) = (y - \text{pred}(\alpha))^2 \quad,
$$
and its surrogate, the least squares loss,
$$
\psi_{\text{LS}}(y, \beta) = \left(y +\alpha_1 - \frac{3}{2} \right)^2 \quad.
$$
As for the case of the least absolute deviation, the only difference between the loss and its surrogate is the presence of the prediction function in the first.

We will now prove that the least squares surrogate is consistent with the squared error: we will first derive a value of $\alpha$, denoted $\underline{\alpha}$ that reaches the Bayes optimal error and then show that the solution to the least squares surrogate agrees in sign with $\underline{\alpha}$ and so also reaches the Bayes optimal error.
\begin{lemma}\label{lemma:bayes_predictor_squared} Let $\underline{\alpha} \in \RR^{k-1}$, be defined component-wise as
$$
\underline{\alpha}(p)_i = i - \left(\sum_{j=1}^k j p_j\right) + \frac{1}{2} \quad.
$$
Then, $\underline{\alpha}$ is a Bayes predictor, that is, $\underline{\alpha} \in \argmin_\alpha L(\alpha, p)$.
\end{lemma}
\begin{proof} Following our proof for the absolute error, we will show that $L(\alpha) - L(\underline{\alpha})$ is always non-negative, which implies that $\underline{\alpha}$ reaches the Bayes optimal error.

Let $r$ and $s$ be defined as $r = \text{pred}(\alpha)$, $s = \text{pred}(\underline{\alpha})$. Then we have the following sequence of equalities:
\begin{equation}
\begin{aligned}\label{eq:excess_risk_square}
L(\alpha) - L(\underline{\alpha}) &= \sum_{i=1}^k p_i ((i - r)^2 - (i - s))^2) \\
&= \sum_{i=1}^k p_i( -2 i r + r^2 + 2 i s - s^2) \qquad \text{ (developing the square)} \\
&= r^2 - s^2 - 2 (r - s)\sum_{i=1}^k i p_i \qquad \text{ (using $\sum_{i=1}^k p_i = 1$)} \\
\end{aligned}
\end{equation}
We will now distinguish three cases: $r > s$, $r < s$ and $r = s$. For each of these cases we will show that $L(\alpha) - L(\underline{\alpha}) \geq 0$, which implies that $\underline{\alpha}$ has a smaller risk than any other $\alpha$, and hence is a Bayes predictor.
\begin{itemize}
\item $s > r$. This implies that $s$ is greater than 1 since $s > r \geq 1$. We can conclude from the definition of prediction function in Eq.~\eqref{eq:pred} that the $(s-1)$-th coordinate of $\underline{\alpha}$ is strictly negative. By the definition of $\underline{\alpha}$ we have that
$
s - 1 - \sum_{i=1}^k i p_i + (1/2) < 0
$ or equivalently $\sum_{i=1}^k i p_i > s - (1/2)$. Using this in Eq.~\eqref{eq:excess_risk_square} we have the following sequence of inequalities:
$$
\begin{aligned}
L(\alpha) - L(\underline{\alpha}) &= r^2 - s^2 - 2 (r - s)\sum_{i=1}^k i p_i  = r^2 - s^2 + \underbrace{(- 2) (r - s)}_{\text{positive}}\sum_{i=1}^k i p_i  \\
&\geq r^2 - s^2 + (- 2) (r - s)(s - \frac{1}{2}) \qquad \text{ (since $\sum_{i=1}^k i p_i \geq s - (1/2)$)} \\
&= \underbrace{r^2 + s^2 - 2 r s}_{(r - s)^2} + r - s =  \underbrace{(r - s)}_{< 0} \underbrace{(r - s + 1)}_{\leq 0} \\
&\geq 0
\end{aligned}
$$
\item $s < r$. We follow a similar argument but reversing the inequalities. The assumption in this case implies that $s$ is smaller than $k$ since $s < r \leq k$. We can conclude from the definition of prediction function in Eq.~\eqref{eq:pred} that the $s$-th coordinate of $\underline{\alpha}$ is positive. By the definition of $\underline{\alpha}$ we have that
$
s - \sum_{i=1}^k i p_i + (1/2) \geq 0
$ or equivalently $-\sum_{i=1}^k i p_i \geq -s - (1/2)$. Using this in Eq.~\eqref{eq:excess_risk_square} we have the following sequence of inequalities:
$$
\begin{aligned}
L(\alpha) - L(\underline{\alpha}) &= r^2 - s^2 - 2 (r - s)\sum_{i=1}^k i p_i  = r^2 - s^2 + \underbrace{2 (r - s)}_{\text{positive}}(-\sum_{i=1}^k i p_i)  \\
&\geq r^2 - s^2 + 2 (r - s)(-s - \frac{1}{2}) \qquad \text{ (since $-\sum_{i=1}^k i p_i \geq -s - (1/2)$)} \\
&= \underbrace{r^2 + s^2 - 2 r s}_{(s - r)^2} + s - r  = \underbrace{(s - r)}_{< 0} \underbrace{(s - r + 1)}_{\leq 0} \\
&\geq 0
\end{aligned}
$$
\item $s = r$. Since the excess risk only depends on $r, s$ and not on the particular values of $\alpha, \underline{\alpha}$ (see Eq.~\eqref{eq:excess_risk_square}), the excess risk is the same in this case, i.e., $L(\alpha) - L(\underline{\alpha}) = 0$.
\end{itemize}
\end{proof}

\begin{theorem}The least squares surrogate $\psi_{\text{LS}}$ is consistent with respect to the squared error.
\end{theorem}
\begin{proof} The proof follows closely that of the least absolute deviation (Theorem~\ref{thm:consistency_lad}) replacing the median by the expected value. In this case, the pointwise error can be written as
$$
A(\alpha, p) = \sum_{i=1}^k p_i \psi_{\text{LS}}(y, \alpha) = \EE_{Y \sim p}\left( Y + \alpha_1 - \frac{3}{2}\right)^2 \quad,
$$
and by the optimality conditions of the expected value, we have that $\alpha_1^* = \EE(\frac{3}{2} - Y|X=x) = \sum_{i=1} p_i (\frac{3}{2} - i)$. Hence, the value for a general $\alpha^*_i$ is given by $i + \alpha_1^*$ (Eq.~\ref{eq:lad_transform}), which can be written as
$$
\alpha^*_i = i  - \sum_{i=1}^k  i p_i + \frac{1}{2} \quad.
$$
This corresponds to the Bayes predictor of Lemma~\ref{lemma:bayes_predictor_squared} and hence is consistent.
\end{proof}

We will encounter another consistent loss function in the Experiments section, when we derive a variant of the AT loss function that is consistent with respect to the square error.
\subsection{Surrogates of the 0-1 loss}\label{scs:zero_one_surrogates}

Perhaps surprisingly, some popular models for ordinal regression turn out to be surrogates, not of the absolute error, but of the 0-1 loss. In this section we focus on the 0-1 loss and we provide a characterization of consistency for the immediate threshold loss function.

{\bfseries Immediate thresholds (IT)}. In this case, the conditional risk can be expressed as
$$
A(\alpha, p) = \sum_{i=1}^k p_j \psi_{\text{IT}}(j, \alpha) = \sum_{i=1}^{k-1} p_{i} \varphi(-\alpha_i) + p_{i+1} \varphi(\alpha_i)
$$

As pointed out by~\citet{Keerthi2003}, and contrary to what happened for AT surrogate, the constraints can not be ignored in general when computing $A^*$. Results that rely on this property such as the excess error bound of Theorem~\ref{thm:excess_risk} will not translate directly for the IT loss. However, we will still be able to characterize the functions $\varphi$ that lead to a consistent surrogate, in a result analogous to Theorem~\ref{thm:all_threshold} for the AT surrogate.

\begin{theorem} \label{thm:consistency_immediate_thresholds}
Let $\varphi$ be convex. Then the IT surrogate is Fisher consistent with respect to the 0-1 loss if and only if $\varphi$ is differentiable at $0$ and $\varphi'(0) < 0$.
\end{theorem}
\begin{proof} As for the AT surrogate, this can be seen as a particular case of Theorem~\ref{thm:main_thm} with $\ell$ the 0-1 loss. We will postpone the proof until Section~\ref{sct:extension_other_loss}.
\end{proof}

\subsection{Extension to other admissible loss functions}\label{sct:extension_other_loss}

In this section we will show that the AT and IT surrogates can be seen as particular instances of a family of loss functions for which we will be able to provide a characterization of consistency.

The admissibility criterion that we require on the loss function is that this is of the form $\ell(i, \alpha) = g\left(\,\abs{i - \text{pred}(\alpha)}\right)$, where $g$ is a non-decreasing function. Intuitively, this condition implies that labels further away from the true label are penalized more than those closer by. This criterion is general enough to contain all losses considered before such as the absolute error, the squared error and (albeit in a degenerate sense) 0-1 loss. A very similar condition is the V-shape property of~\citep{Li2007}. This property captures the notion that the loss should not decrease as the predicted value moves away from the true value by imposing that $\ell$ verifies $\ell(y, \alpha) \leq \ell(y, \alpha')$ for $\abs{y - \text{pred}(\alpha)} \leq \abs{y - \text{pred}(\alpha')}$. The only difference between the two conditions is that our admissibility criterion adds a symmetric condition, i.e., $\ell$ verifies that the loss of predicting $a$ when the true label is $b$ is the same as the loss of predicting $b$ when the true label is $a$, which is not necessarily true for V-shaped loss functions. We conjecture that the results in this section are valid for general V-shaped loss functions, although for simplicity we have only proven results for symmetric V-shaped loss functions.
For the rest of this section, we will consider that $\ell$ is a loss function that verifies the admissibility criterion.

We define $c_i$ by $c_0 = g(0)$ and $c_i = g(i) - g(i-1)$ for $0<i\geq k$. With this notation it is easy to verify by induction that $g$ can be written as a sum of $c_i$ with the formula $g(i) = \sum_{j=1}^{i} c_j$ and that $c_i \geq 0$ by the admissibility property. Following the same development as in Eq.~\eqref{eq:development_absolute_error},
any admissible loss function can be written as a sum of $c_i$ as
\begin{equation}\label{eq:general_loss}
\ell(y, \alpha) = g\left( \sum_{i=1}^{y-1}\llbracket  \alpha_i \geq 0 §\rrbracket + \sum_{i=y}^{k-1}\llbracket  \alpha_i < 0 §\rrbracket \right) = \sum_{i=1}^{y-1} c_{y-i} \llbracket  \alpha_i \geq 0 §\rrbracket + \sum_{i=y}^{k-1} c_{i-y+1}\llbracket  \alpha_i < 0 §\rrbracket \quad.
\end{equation}

In light of this, it seems natural to define a surrogate for this general loss function by replacing the 0-1 loss with a surrogate as the hinge or logistic that we will denote by $\varphi$. This defines a new surrogate that we will denote \emph{generalized all threshold (GAT)}:
$$
\psi_{\mathrm{GAT}}(y, \alpha) := \sum_{i=1}^{y-1} \varphi(-\alpha_i) c_{y-i} + \sum _{i=y}^{k-1} \varphi(\alpha_i) c_{i - y+1} \quad.
$$

In the special case of the absolute error, $c_i$ is identically equal to 1 and we recover AT loss of Eq.~\eqref{eq:def_all_thresh}. Likewise, for the zero-one loss, $c_i$ will be one for $i \in \{y - 1, y\}$ and zero otherwise, recovering the IT loss of Eq.~\eqref{eq:immediate_threshold}. We will now present the main result of this section, which has Theorems~\ref{thm:all_threshold} and~\ref{thm:consistency_immediate_thresholds} as particular cases.

\begin{theorem}\label{thm:main_thm}
   Let $\varphi$ be convex. Then the GAT surrogate is consistent if and only if $\varphi$ is differentiable at $0$ and $\varphi'(0) < 0$.
\end{theorem}

Before presenting the proof of this theorem, we will need some auxiliary results. Unlike for the absolute error, in this case we will not be able to derive a closed form of the optimal decision function. However, we will be able to derive a formula for the excess risk in terms of the functions $u, v: \Delta^k \to \RR^{k-1}$, defined as
$$
u_i(p) = \sum_{j=1}^i p_j c_{i-j+1} \quad
v_i(p) = \sum_{j=i+1}^{k} p_j c_{j-i} \quad.
$$
Note that we have overloaded the function $u$ defined in Section~\ref{sct:absolute_error_surrogates}. This is not a coincidence, as when $\ell$ is the absolute error both definitions coincide. Using this notation, the surrogate risk can be conveniently written as
$$
\begin{aligned}
A(\alpha, p) &= \sum_{i=1}^k p_i \psi_{\mathrm{GAT}}(i, \alpha) = \sum_{i=1}^k p_i \left(\sum_{j=1}^{i-1} \varphi(-\alpha_j) c_{i-j} + \sum_{j=i}^{k-1} \varphi(\alpha_j) c_{j-i+1} \right) \\
&= \sum_{i=1}^{k-1} v_i(p) \varphi(-\alpha_{i}) + u_i(p) \varphi(\alpha_{i})  \quad,
\end{aligned}
$$
and we have the following formulas for the excess risk:

\begin{lemma}\label{lemma:aux_main_thm} Let ${p \in \Delta^k}$, $\alpha \in \mathcal{S}$, $r = \text{pred}(\alpha)$ and $r^*$ be the label predicted by any Bayes decision function at $p$. Then, it is verified that
  $$
  L(\alpha, p) - L^*(p) = \begin{dcases}
  \sum_{i=r}^{r^*} (v_i(p) - u_i(p)) \quad \text{ if } r < r^* \quad \\
  \sum_{i=r^*}^{r-1} (u_i(p) - v_i(p))  \quad \text{ if } r > r^* \quad.\\
  \end{dcases}
  $$
\end{lemma}

\begin{proof}
  The risk can be expressed in terms of $u_i$ and $v_i$ (where the dependence of $p$ is implicit) as
  \begin{equation}\label{eq:generic_risk}
  \begin{aligned}
  L(\alpha) &= \sum_{i=1}^{k} p_i g(\,\abs{r - i}) = \sum_{i=1}^{r-1} p_i g(r - i) + \sum_{i=r+1}^{k} p_i g(i - r) \\
  &=  \sum_{i=1}^{r-1} p_i \sum_{j=1}^{r-i} c_j + \sum_{i=r+1}^{k} p_i \sum_{j=1}^{i - r} c_j \\
  &= \sum_{i=1}^{r-1} u_i + \sum_{i=r}^{k-1} v_i \quad ,
  \end{aligned}
\end{equation}
  hence for $r < r^*$,
  $$
  0 \leq L(\alpha) - L^* = \sum_{i=1}^{r-1} u_i + \sum_{i=r}^{k-1} v_i - \left(\sum_{i=1}^{r^*-1} u_i + \sum_{i=r^*}^{k-1} v_i \right) = \sum_{i=r}^{r^*-1} (v_i - u_i) \quad,
  $$
  and similarly for $r > r^*$
  $$
  0 \leq L(\alpha) - L^* = \sum_{i=1}^{r-1} u_i + \sum_{i=r}^{k-1} v_i - \left(\sum_{i=1}^{r^*-1} u_i + \sum_{i=r^*}^{k-1} v_i \right) = \sum_{i=r^*}^{r-1} (u_i - v_i) \quad,
  $$
\end{proof}

\begin{proof}[Proof of Theorem \ref{thm:main_thm}]
  This proof loosely follows the steps by~\citet[Theorem 6]{Bartlett2003}, with the difference that we must ensure that the optimal value of the surrogate risk lies within $\mathcal{S}$ and adapted to consider multiple classes. We denote by $\alpha^*$ the value in $\mathcal{S}$ that minimizes $A(\cdot)$, by $r$ the prediction at $\alpha^*$ and by $r^*$ the prediction of any Bayes decision function, where the dependence on $p$ is implicit.

  $(\implies)$ We first prove that consistency implies $\varphi$ is differentiable at $0$ and $\varphi'(0) < 0$. We do so by proving that the subdifferential at zero is reduced to a single vector. Since $\varphi$ is convex, we can find subgradients $g_1 \geq g_2$ of $\varphi$ at zero  such that, for all $\beta \in \RR$
  $$
  \begin{aligned}
    \varphi(\beta) &\geq g_1 \beta + \varphi(0) \\
    \varphi(\beta) &\geq g_2 \beta + \varphi(0) \quad.\\
  \end{aligned}
  $$
  Then we have for all $i$
  \begin{equation}\label{eq:A_calibrated}
  \begin{aligned}
    v_i \varphi(-\beta) + u_i \varphi(\beta) &\geq v_i (g_1 \beta + \varphi(0)) + u_i (- g_2 \beta + \varphi(0)) \\
    &= (v_i g_1 - u_i g_2) \beta + (v_i + u_i) \varphi(0) \\
    &= \beta \left( \frac{1}{2}(v_i+u_i) (g_1 - g_2) + \frac{1}{2}(v_i-u_i)(g_1 + g_2)\right) + (v_i + u_i) \varphi(0)\quad.
  \end{aligned}
\end{equation}
For $0 < \varepsilon < 1/2$, we will consider the following vector of conditional probabilities
$$
p = \left(0, \cdots, 0, \frac{1}{2} - \varepsilon, \frac{1}{2} + \varepsilon \right) \quad,
$$
from where $u_i$ and $v_i$ take the following simple form
$$
u_i = \begin{cases}
p_{k-1} c_1 &\text{ if } i = k-1 \\
0 &\text{ otherwise }
\end{cases}, \quad
v_i = \begin{cases}
p_{k} c_1 &\text{ if } i = k-1 \\
p_{k-1} c_{k-i-1} + p_{k} c_{k-i} &\text{ otherwise }
\end{cases}
$$
hence by Eq.~\eqref{eq:generic_risk} consistency implies $r=k$ and so we must have $\alpha_{k-1}^* < 0$.

Let now $\tilde{\alpha} \in \mathcal{S}$ be a vector that equals $\alpha^*$ in all except the last component, which is zero (i.e., $\tilde{\alpha}_{k-1} = 0$). We will now prove that if $g_1 > g_2$ then $A(\tilde{\alpha}, p) < A(\alpha^*, p)$ leading to a contradiction. For the particular choice of $u_{k-1}, v_{k-1}$ above, equation~\eqref{eq:A_calibrated} can be simplified to
$$v_{k-1} \varphi(-\beta) + u_{k-1} \varphi(\beta) \geq \beta \left[ \frac{1}{2} (g_1 - g_2) + \varepsilon(g_1 + g_2)\right] + (v_{k-1} + u_{k-1}) \varphi(0)\quad.
$$
Since by assumption $g_1 > g_2$, it is always possible to choose $\varepsilon$ small enough such that the quantity inside the square brackets is strictly positive. Special casing at $\beta = \alpha^*_{k-1}$ and using $\alpha^*_{k-1} < 0$ yields the following inequality:
$$v_{k-1} \varphi(-\alpha^*_{k-1}) + u_{k-1} \varphi(\alpha^*_{k-1}) \geq  (v_{k-1} + u_{k-1}) \varphi(0)\quad.
$$
We then have the following sequence of inequalities:
\begin{equation*}
\begin{aligned}
A(\alpha^*, p) &= \sum_{i=1}^{k-1} v_i(p)\varphi(-\alpha^*_i) + u_i \varphi(\alpha^*_i) \\
&\geq \sum_{i=1}^{k-2} \left\{v_i(p)\varphi(-\alpha^*_i) + u_i \varphi(\alpha^*_i)\right\} + (v_{k-1} - u_{k-1})\alpha^*_{k-1} + (v_{k-1} + u_{k-1})\varphi(0) \\
&\qquad \text{ (by last inequality)}\\
&> \sum_{i=1}^{k-2} \left\{v_i(p)\varphi(-\alpha^*_i) + u_i \varphi(\alpha^*_i)\right\} + v_{k-1}\varphi(0) + u_{k-1}\varphi(0) \\
&= A(\tilde{\alpha}, p)\quad,
\end{aligned}
\end{equation*}
which results in $A(\alpha^*, p) > A(\tilde{\alpha}, p)$, contradiction since $\alpha^*$ is the value with lowest surrogate risk. This implies that if the GAT loss is consistent, then $\varphi$ is differentiable at $0$. To see that we must also have $\varphi'(0) < 0$, notice that from Eq.~\eqref{eq:A_calibrated} we have
  $$
  A_i(\beta) \geq (v_i - u_i) \varphi'(0) \beta + A_i(0) \quad.
  $$
  But for any $v_i > u_i$ and $\beta < 0$, if $\varphi'(0) \geq 0$, then this expression is greater than $A_i(0)$. Hence, if GAT is consistent then $\varphi'(0) < 0$, which concludes one of the implications of the proof.

\hfill

  $(\impliedby)$ We now prove that if $\varphi$ is differentiable at $0$ and $\varphi'(0) < 0$, then GAT is consistent.

The first order optimality conditions states that there exists $\lambda_i \geq 0$ such that the optimal value of $A(\alpha, p)$ subject to $\alpha \in \mathcal{S}$ is the minimizer of the following unconstrained function:
  $$
  G(\alpha) = A(\alpha, p) + \sum_{i=1}^{k-1}\lambda_i (\alpha_{i} - \alpha_{i+1}) \quad.
  $$
  We show that assuming GAT is not consistent (i.e., $L(\alpha^*) > L^*$) leads to a contradiction and hence GAT must be consistent.

We start by computing the partial derivative of $G$ at zero:
  $$
  \frac{\partial G}{\partial \alpha_{i}}\Bigr|_{\alpha_{i} = 0} =  (u_{i} - v_{i})\varphi'(0) - \lambda_{i-1} + \lambda_{i} \quad,
  $$
  were for convenience $\lambda_0 = 0$.  Note that $\alpha^*_{r-1} < 0 \leq \alpha^*_r$ verifies by definition of prediction function. Hence, the inequality constraints are verified with strict inequality and by complementary slackness $\lambda_{r-1} = 0$. Suppose first $r < r^*$. Then, the addition of all partial derivatives between $r$ and $r^*$ yields
  $$
    \sum_{i = r}^{r^*} \frac{\partial G}{\partial \alpha_i}\Bigr|_{\alpha_i = 0} =
    \left(\sum_{i=r}^{r^*} u_i - v_i \right) \varphi'(0) + \lambda_{r^*} \quad,
  $$
   which by Lemma~\ref{lemma:aux_main_thm} is strictly positive. Consider the convex real-valued function of $\alpha_i \to G(\alpha_1^*, \ldots, \alpha_{r-1}^*, \ldots, \alpha_i, \ldots)$, that is, the function $G(\alpha^*)$ restricted to $\alpha_i, i \geq r$. Since $\alpha^*_{i} \geq 0$ for all $i \geq r$ by the definition of prediction function, and knowing the subdifferential of a convex function is a monotonous, this implies that $\partial G / \partial \alpha_{i\geq r} \leq 0$ at $\alpha_i = 0$, contradiction.

  Now we suppose $r > r^*$. Then, the addition of all partial derivatives between $r^*$ and $r$ yields
$$
    \sum_{i = r^*}^r \frac{\partial G}{\partial \alpha_i}\Bigr|_{\alpha_i = 0} =\left(\sum_{i=r^*}^{r} u_i - v_i \right) \varphi'(0) - \lambda_{r^*-1}\quad,
$$
   which by Lemma~\ref{lemma:aux_main_thm} is strictly negative. As before, we consider the function the function $G(\alpha^*)$ restricted to $\alpha_i, i \geq r$. By the definition of prediction function, $\alpha^*_{i < r} < 0$ and so the fact that the subdifferential of a real-valued function is monotonous, we have that $\partial G / \partial \alpha_{i\geq r} \geq 0$ at $\alpha_i = 0$, contradiction.
\end{proof}

\section{Threshold-based decision functions and parametric consistency}\label{scs:parametric_consistency}
In this section we revisit the assumption that the optimal decision function can be estimated independently at each point $x \in \mathcal{X}$. This is implicitly assumed on most consistency studies, however in practice models often enforce inter-observational constraints (e.g. smoothness). In the case of ordinal regression it is often the case that the decision functions are of the form
\begin{equation}\label{eq:thresholds}
f(x) = (\theta_1 - g(x), \theta_2 - g(x), \ldots, \theta_{k-1} - g(x)) \quad,
\end{equation}
where $(\theta_1, \ldots, \theta_{k-1})$ is a non-decreasing vector (i.e., its components form a non-decreasing sequence) known as the \emph{vector of thresholds} (hence the appearance of the name thresholds in many models) and $g$ is a measurable function. We will call decision functions of this form \emph{threshold-based} decision functions. All the examined models are of this form with the exception of the least absolute deviation are commonly constrained to this family of decision functions~\citep{Keerthi2003,Rennie,lin2006large,Shashua}.

The main issue with such decision functions is that since the vector of thresholds is estimated from the data, it is no longer true that the optimal decision function can be estimated independently at each point. This implies that the pointwise characterization of Fisher consistency described in Lemma~\ref{lemma:characterization_Fisher} does no longer hold when restricted to this family and hence the consistency proofs in previous sections no longer hold.

Let $\mathcal{F}$ be the set of functions of the form of Eq.~\eqref{eq:thresholds}. We will now apply the notion of \emph{$\mathcal{F}$-consistency}  or \emph{parametric consistency} of~\citep{shi2015hybrid} to the threshold-based setting. This is merely the notion of Fisher consistency where the decision functions are restricted to a family of interest:

\begin{definition}
  ($\mathcal{F}$-{\bfseries Consistency}) Given a surrogate loss function $\psi: \mathcal{Y}\times \mathcal{S} \to \RR$, we will say that the surrogate loss function $\psi$ is $\mathcal{F}$-consistent with respect to the loss $\ell: \mathcal{Y}\times \mathcal{S} \to \mathcal{Y}$ if for every probability distribution over $X \times Y$ it is verified that every minimizer of the $\psi$-risk reaches the optimal risk in $\mathcal{F}$, that is,
  $$f^* \in \argmin_{f \in \mathcal{F}} \mathcal{A}(f) \implies \mathcal{L}(f^*) = \inf_{f \in \mathcal{F}}\mathcal{L}(f) \quad .$$
\end{definition}

We will show that by imposing additional constraints on the probability distribution $P$ we will be able to derive $\mathcal{F}$-consistency for particular surrogates. In the following theorem we give one sufficient condition for $\mathcal{F}$-consistency of the logistic all threshold and the logistic CL. This condition involves the odds-ratio, defined as
$$
R_i(x) = \frac{u_i(\eta(x))}{1 - u_i(\eta(x))} \quad,
$$
where $\eta(x)$ is the vector of conditional probabilities defined in Section~\ref{scs:full_conditional_risk} and $u_i$ is the vector of conditional probabilities.

\begin{theorem}
  If the quotient of successive $R_i$ is independent of $x$ for all $0<i<k-1$, that is, if
  $$
  \frac{R_i(x)}{R_{i+1}(x)} = a_i \quad \forall x \in \mathcal{X} \quad,
  $$
  for some real number $a_i$,
  then the logistic all threshold and the logistic CL are $\mathcal{F}$-consistent.
\end{theorem}
\begin{proof}
  It will be sufficient to prove that under the constraints on $P$, the optimal decision function for the unconstrained problem belongs to $\mathcal{F}$. In Section~\ref{sct:absolute_error_surrogates} we derived the optimal decision function for the logistic all threshold and the logistic CL. Hence, we can write
  $$
  \alpha^*_i(\eta(x)) - \alpha^*_{i+1}(\eta(x)) = \log\left( \frac{u_i(\eta(x))}{1 - u_i(\eta(x))}\right) - \log\left( \frac{u_{i+1}(\eta(x))}{1 - u_{i+1}(\eta(x))}\right) = \log \left(\frac{R_i(x)}{R_{i+1}(x)}\right) \quad.
  $$
  It is clear that $\alpha$ is a threshold-based decision function of the form Eq.~\eqref{eq:thresholds} if and only if $\alpha_i - \alpha_{i+1}$ does not depend on $x \in \mathcal{X}$. Given the above, we can guarantee that the optimal $\alpha$ belongs to $\mathcal{F}$ if the last term is independent of $x$. One can then easily recognize the quotient of odds-ratio of the theorem's conditions and conclude that $\alpha^*_i(\eta(x)) - \alpha^*_{i+1}(\eta(x))$ is independent of $x$. This concludes the proof.
\end{proof}

This sufficient condition is admittedly a very stringent one on the probability distribution $P$. Unfortunately, a deeper understanding of $\mathcal{F}$-consistency, while an interesting future direction, is outside the scope of the current paper.

\section{Experiments: A novel surrogate for the squared error}\label{scs:experiments}

While the focus of this work is a theoretical investigation of consistency, we have also conducted experiments to study a novel surrogate suggested by the results of the last section. There, we constructed a surrogate that is consistent with any loss function that verifies an admissible criterion. In particular, Theorem~\ref{thm:main_thm} applied to the squared loss yields the following consistent surrogate:
$$
\psi(y, \alpha) = \sum_{i=1}^{y-1} \varphi(-\alpha_i)(2 (y - i) - 1) + \sum_{i=y}^{k-1} \varphi(\alpha_i) (2 (i - y) + 1) \quad.
$$
In principle, any binary loss function can be used for $\varphi$, although in the experiments we set it to the hinge loss function.
To the best of our knowledge, this is a novel surrogate. We compare the cross-validation error of this surrogate on different datasets
 against the least squares surrogate (for which we proved consistency in~\S\ref{scs:squared_error})
where $\beta \in \RR$ and prediction is given by rounding to the closest integer.
In both cases, we consider linear decision functions, i.e.
$$
\alpha = (\theta_1 - \langle w, x \rangle, \ldots, \theta_{k-1} - \langle w, x \rangle) \quad \text{ and } \quad \beta = \langle w, x \rangle \quad.
$$
In each case, the optimal values of $w, \theta$ were found by minimizing the empirical surrogate risk. For the training sample $\{(x_1, y_1), \ldots, (x_n, y_n)\}$, $x_i \in \RR^p$, it yielding the following optimization problems for GAT and least squares, respectively:
\begin{gather*}
\argmin_{\theta \in \mathcal{S}, w \in \RR^p} \sum_{i=1}^n\left\{ \sum_{j=1}^{y_i-1} \varphi(\langle w, x_i \rangle -\theta_j)(2 (y - j) - 1) + \sum_{j=y_i}^{k-1} \varphi(\theta_j-\langle w, x_i \rangle) (2 (j - y) + 1) \right\} \\
\argmin_{w \in \RR^p} \sum_{i=1}^n (y_i - \langle w, x_i \rangle)^2
\end{gather*}

The different datasets that we will consider are described in~\citep{Keerthi2003} and can be download from the authors website\footnote{\url{http://www.gatsby.ucl.ac.uk/~chuwei/ordinalregression.html}. }. We display results for the 9 datasets of SET I using the version of the datasets with 5 bins, although similar results were observed when using the datasets with 10 bins. Given the small dimensionality of the datasets (between 6 and 60) and the comparatively high number of samples (between 185 and 4000), we did not consider the use of regularization.

Performance is computed as the squared error on left out data, averaged over 20 folds.
We report this performance in Figure~\ref{fig:results_GAT}, where it can be seen that the GAT surrogate outperforms LS on 7 out of 9 datasets, although admittedly the difference is small and only statistically significant on 3 datasets. However, this shows that previous theoretical results can be used to generate consistent surrogates that are competitive in a practical scenario.

\begin{figure}[h]
\center \includegraphics[width=0.8\linewidth]{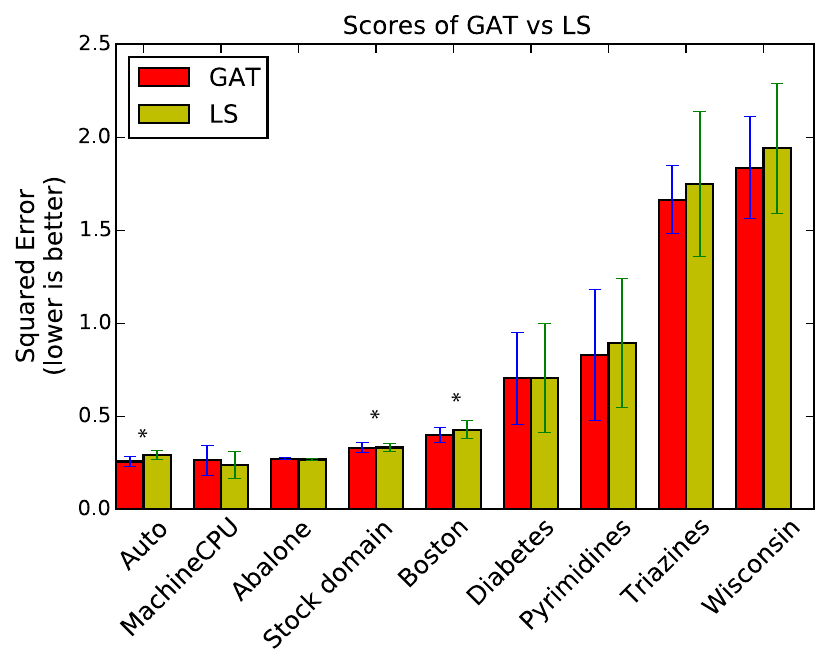}
\caption{Scores of the generalized all threshold (GAT) and least squares (LS) surrogate on 6 different datasets. The scores are computed as the squared error between the prediction and the true labels on left out data, averaged over 20 cross-validation folds.
On 7 out of 9 datasets all the GAT surrogate outperforms the least squares estimator, showing that this surrogate yields a highly competitive model. Datasets for which a Wilcoxon signed-rank test rejected the null hypothesis that the means are equal with $p$-value $< 0.01$ are highlighted by a $*$ symbol over the bars.}\label{fig:results_GAT}
\end{figure}

\section{Conclusions}

In this paper we have characterized the consistency for a rich family of surrogate loss functions used for ordinal regression. Our aim is to bridge the gap between the consistency properties known for classification and ranking and those known for ordinal regression.

We have first described a wide family of ordinal regression methods under the same framework. The surrogates of the absolute error that we have considered are the all threshold (AT), cumulative link (CL), and least absolute deviation (LAD), while the surrogate for the 0-1 loss is the immediate threshold (IT).

For all the surrogates considered, we have characterized its consistency. For AT and IT, consistency was characterized by the derivative of a real-valued convex function at zero (Theorems~\ref{thm:all_threshold} and \ref{thm:consistency_immediate_thresholds} respectively). For CL, consistency was characterized by a simple condition on its link function (Theorem~\ref{thm:consistency_cl}) and for LAD we have extended the proof of~\citet{Ramaswamy2012} to an arbitrary number of classes (Theorem~\ref{thm:all_threshold}). Furthermore, we have proven that AT verifies a decomposability property and using this property we have provided excess risk bounds that generalize those of~\citet{Bartlett2003} for binary classification(Theorem~\ref{thm:excess_risk}).

The derivation we have given when introducing IT and AT are identical except for the underlying loss function. This suggest that both can be seen as special cases of a more general family of surrogates. In Section~\ref{sct:extension_other_loss} we have constructed such surrogate and characterized its consistency with respect to a general loss function that verifies an admissibility condition. Again, the characterization only relies on the derivative at zero of a convex real-valued function. We named this surrogate generalized all threshold (GAT).

In Section~\ref{scs:parametric_consistency} we have turned back to examine one of the assumptions described in the introduction and that is common to the vast majority of consistency studies, i.e., that the optimal decision function can be estimated independently at every sample. However, in the setting of ordinal regression it is common for decision functions to have a particular structure known as threshold-based decision functions and which violates this assumption.
Following~\citep{shi2015hybrid}, we are able to prove a restricted notion of consistency known as $\mathcal{F}$-consistency or parametric consistency on two surrogates by enforcing constraints on the probability distribution $P$. We believe this restricted notion of consistency to be important in practice and we look forward to see consistency studies extended to consider different types of decision functions, such as smooth functions, polynomial functions, etc.

Finally, in Section~\ref{scs:experiments} we provide an empirical comparison for the GAT surrogate. The underlying loss function that we consider in this case is the squared error, in which case GAT yields a novel surrogate. We compare this surrogate against the least squares surrogate in terms of cross-validation error. Our results show that GAT outperforms least squares on 7 out of 9 datasets, showing the pertinence of such surrogate on real-world datasets.

A direction for future work would be to study the related notion of asymptotic consistency for ordinal regression loss functions, similar to the results that already exist for binary classification~\citep{devroye1994strong,Steinwart2002}.

\section{Acknowledgments}
Most of this work was done while FP was a PhD student at the Parietal project-team at INRIA Saclay and acknowledges financial support under grants IRMGroup ANR-10-BLAN-0126-02 and BrainPedia ANR-10-JCJC 1408-01 and the ``Chaire \'Economie des Nouvelles Donn\'ees'', under
the auspices of Institut Louis Bachelier, Havas-Media and Universit\'e Paris-Dauphine.
We would like to thank our colleague Guillaume Obozinski for fruitful discussions.

\bibliography{biblio}

\begin{thebibliography}{44}
\providecommand{\natexlab}[1]{#1}
\providecommand{\url}[1]{\texttt{#1}}
\expandafter\ifx\csname urlstyle\endcsname\relax
  \providecommand{\doi}[1]{doi: #1}\else
  \providecommand{\doi}{doi: \begingroup \urlstyle{rm}\Url}\fi

\bibitem[Agarwal(2008)]{Agarwal2008}
Shivani Agarwal.
\newblock Generalization bounds for some ordinal regression algorithms.
\newblock \emph{ALT '08 Proceedings of the 19th international conference on
  Algorithmic Learning Theory}, 2008.

\bibitem[Agresti(2010)]{agresti2010analysis}
Alan Agresti.
\newblock \emph{Analysis of ordinal categorical data}.
\newblock John Wiley \& Sons, 2010.

\bibitem[Ananth and Kleinbaum(1997)]{ananth1997regression}
Cande~V. Ananth and David~G. Kleinbaum.
\newblock Regression models for ordinal responses: a review of methods and
  applications.
\newblock \emph{International journal of epidemiology}, 1997.

\bibitem[Armstrong and Sloan(1989)]{ARMSTRONG01011989}
Ben~G. Armstrong and Margaret Sloan.
\newblock Ordinal regression models for epidemiologic data.
\newblock \emph{American Journal of Epidemiology}, 1989.

\bibitem[{\'A}vila~Pires et~al.(2013){\'A}vila~Pires, Szepesvari, and
  Ghavamzadeh]{avila2013cost}
Bernardo {\'A}vila~Pires, Csaba Szepesvari, and Mohammad Ghavamzadeh.
\newblock Cost-sensitive multiclass classification risk bounds.
\newblock In \emph{Proceedings of The 30th International Conference on Machine
  Learning}, 2013.

\bibitem[Bartlett et~al.(2003)Bartlett, Jordan, and Mcauliffe]{Bartlett2003}
Peter~L. Bartlett, Michael~I. Jordan, and Jon~D. Mcauliffe.
\newblock Convexity, classification, and risk bounds.
\newblock \emph{Journal of the American Statistical Association}, 2003.

\bibitem[Ben-David et~al.(2003)Ben-David, Eiron, and Long]{ben2003difficulty}
Shai Ben-David, Nadav Eiron, and Philip~M. Long.
\newblock On the difficulty of approximately maximizing agreements.
\newblock \emph{Journal of Computer and System Sciences}, 2003.

\bibitem[Boucheron et~al.(2005)Boucheron, Bousquet, and
  Lugosi]{boucheron2005theory}
St{\'e}phane Boucheron, Olivier Bousquet, and G{\'a}bor Lugosi.
\newblock Theory of classification: A survey of some recent advances.
\newblock \emph{ESAIM: probability and statistics}, 9:\penalty0 323--375, 2005.

\bibitem[Boyd and Vandenberghe(2004)]{boyd2004convex}
Stephen~P. Boyd and Lieven Vandenberghe.
\newblock \emph{Convex optimization}.
\newblock Cambridge University Press, 2004.

\bibitem[Buja et~al.(2005)Buja, Stuetzle, and Shen]{buja2005loss}
Andreas Buja, Werner Stuetzle, and Yi~Shen.
\newblock Loss functions for binary class probability estimation and
  classification: Structure and applications.
\newblock Technical report, University of Pennsylvania, November 2005.

\bibitem[Calauzenes et~al.(2012)Calauzenes, Usunier, Gallinari,
  et~al.]{calauzenes2012non}
Cl{\'e}ment Calauzenes, Nicolas Usunier, Patrick Gallinari, et~al.
\newblock On the (non-) existence of convex, calibrated surrogate losses for
  ranking.
\newblock In \emph{Advances in Neural Information Processing Systems (NIPS)},
  2012.

\bibitem[Chu and Ghahramani(2004)]{Chu2005a}
Wei Chu and Zoubin Ghahramani.
\newblock Gaussian processes for ordinal regression.
\newblock \emph{Journal of Machine Learning Research}, 2004.

\bibitem[Chu and Keerthi(2005)]{Keerthi2003}
Wei Chu and S~Sathiya Keerthi.
\newblock New approaches to support vector ordinal regression.
\newblock In \emph{Proceedings of the 22th International Conference on Machine
  Learning (ICML)}, 2005.

\bibitem[Ciliberto et~al.(2016)Ciliberto, Rosasco, and
  Rudi]{ciliberto2016consistent}
Carlo Ciliberto, Lorenzo Rosasco, and Alessandro Rudi.
\newblock A consistent regularization approach for structured prediction.
\newblock In \emph{Advances in Neural Information Processing Systems}, 2016.

\bibitem[Crammer and Singer(2001)]{crammer2001pranking}
Koby Crammer and Yoram Singer.
\newblock Pranking with ranking.
\newblock In \emph{Advances in Neural Information Processing Systems (NIPS)},
  2001.

\bibitem[Crammer and Singer(2005)]{crammer2005online}
Koby Crammer and Yoram Singer.
\newblock Online ranking by projecting.
\newblock \emph{Neural Computation}, 2005.

\bibitem[Devroye et~al.(1994)Devroye, Gyorfi, Krzyzak, and
  Lugosi]{devroye1994strong}
Luc Devroye, Laszlo Gyorfi, Adam Krzyzak, and Gabor Lugosi.
\newblock On the strong universal consistency of nearest neighbor regression
  function estimates.
\newblock \emph{The Annals of Statistics}, 1994.

\bibitem[Doyle et~al.(2013)Doyle, Ashburner, Zelaya, Williams, Mehta, and
  Marquand]{doyle2013multivariate}
Orla~M. Doyle, John Ashburner, F.O. Zelaya, Stephen~C.R. Williams, Mitul~A.
  Mehta, and Andre~F. Marquand.
\newblock Multivariate decoding of brain images using ordinal regression.
\newblock \emph{NeuroImage}, 2013.

\bibitem[Duchi et~al.(2010)Duchi, Mackey, and Jordan]{Duchi2010}
John~C. Duchi, Lester~W. Mackey, and Michael~I. Jordan.
\newblock On the consistency of ranking algorithms.
\newblock In \emph{Proceedings of the 27th International Conference on Machine
  Learning (ICML)}, 2010.

\bibitem[Feldman et~al.(2012)Feldman, Guruswami, Raghavendra, and
  Wu]{feldman2012agnostic}
Vitaly Feldman, Venkatesan Guruswami, Prasad Raghavendra, and Yi~Wu.
\newblock Agnostic learning of monomials by halfspaces is hard.
\newblock \emph{SIAM Journal on Computing}, 41\penalty0 (6):\penalty0
  1558--1590, 2012.

\bibitem[Gneiting and Raftery(2007)]{gneiting2007strictly}
Tilmann Gneiting and Adrian~E Raftery.
\newblock Strictly proper scoring rules, prediction, and estimation.
\newblock \emph{Journal of the American Statistical Association}, 102\penalty0
  (477):\penalty0 359--378, 2007.

\bibitem[Greene(1997)]{greene1997econometric}
William~H. Greene.
\newblock \emph{Econometric analysis}.
\newblock Prentice-Hall, Inc, 1997.

\bibitem[Hartrick et~al.(2003)Hartrick, Kovan, and Shapiro]{PAPR:PAPR3034}
Craig~T. Hartrick, Juliann~P. Kovan, and Sharon Shapiro.
\newblock The numeric rating scale for clinical pain measurement: A ratio
  measure?
\newblock \emph{Pain Practice}, 3\penalty0 (4):\penalty0 310--316, 2003.

\bibitem[Herbrich et~al.(1999)Herbrich, Graepel, and Obermayer]{herbrich1999}
Ralf Herbrich, Thore Graepel, and Klaus Obermayer.
\newblock Support vector learning for ordinal regression.
\newblock \emph{IET Conference Proceedings}, 1999.

\bibitem[Knuth(1992)]{knuth1992two}
Donald~E Knuth.
\newblock Two notes on notation.
\newblock \emph{American Mathematical Monthly}, 1992.

\bibitem[Kramer et~al.(2001)Kramer, Widmer, Pfahringer, and
  De~Groeve]{kramer2001prediction}
Stefan Kramer, Gerhard Widmer, Bernhard Pfahringer, and Michael De~Groeve.
\newblock Prediction of ordinal classes using regression trees.
\newblock \emph{Fundamenta Informaticae}, 2001.

\bibitem[Lee et~al.(2004)Lee, Lin, and Wahba]{lee2004multicategory}
Yoonkyung Lee, Yi~Lin, and Grace Wahba.
\newblock Multicategory support vector machines: Theory and application to the
  classification of microarray data and satellite radiance data.
\newblock \emph{Journal of the American Statistical Association}, 2004.

\bibitem[Li and Lin(2007)]{Li2007}
Ling Li and Hsuan-tien Lin.
\newblock Ordinal regression by extended binary classification.
\newblock In \emph{Advances in Neural Information Processing Systems (NIPS)}.
  MIT Press, 2007.

\bibitem[Lin and Li(2006)]{lin2006large}
Hsuan-Tien Lin and Ling Li.
\newblock Large-margin thresholded ensembles for ordinal regression: Theory and
  practice.
\newblock In \emph{Algorithmic Learning Theory}. Springer, 2006.

\bibitem[Lin(2004)]{Lin2004}
Yi~Lin.
\newblock A note on margin-based loss functions in classification.
\newblock \emph{Statistics \& Probability Letters}, 2004.

\bibitem[Liu(2011)]{liu2011learning}
Tie-Yan Liu.
\newblock \emph{Learning to rank for information retrieval}.
\newblock Springer Science \& Business Media, 2011.

\bibitem[McCullagh(1980)]{McCullagh1980}
Peter McCullagh.
\newblock Regression models for ordinal data.
\newblock \emph{Journal of the Royal Statistical Society}, 1980.

\bibitem[Osokin et~al.(2017)Osokin, Bach, and
  Lacoste-Julien]{osokin2017structured}
Anton Osokin, Francis Bach, and Simon Lacoste-Julien.
\newblock On structured prediction theory with calibrated convex surrogate
  losses.
\newblock \emph{arXiv preprint arXiv:1703.02403}, 2017.

\bibitem[Ramaswamy and Agarwal(2012)]{Ramaswamy2012}
Harish~G. Ramaswamy and Shivani Agarwal.
\newblock {Classification Calibration Dimension for General Multiclass Losses}.
\newblock In \emph{Advances in Neural Information Processing Systems (NIPS)}.
  MIT Press, 2012.

\bibitem[Ramaswamy and Agarwal(2016)]{ramaswamy2014convex}
Harish~G Ramaswamy and Shivani Agarwal.
\newblock Convex calibration dimension for multiclass loss matrices.
\newblock \emph{Journal of Machine Learning Research}, 2016.

\bibitem[Reid and Williamson(2010)]{Reid2010}
Mark~D Reid and Robert~C Williamson.
\newblock Composite binary losses.
\newblock \emph{Journal of Machine Learning Research}, 2010.

\bibitem[Rennie and Srebro(2005)]{Rennie}
Jason D.~M. Rennie and Nathan Srebro.
\newblock Loss functions for preference levels : Regression with discrete
  ordered labels.
\newblock In \emph{Proceedings of the IJCAI Multidisciplinary Workshop on
  Advances in Preference Handling}, 2005.

\bibitem[Shashua and Levin(2003)]{Shashua}
Amnon Shashua and Anat Levin.
\newblock Ranking with large margin principle : Two approaches.
\newblock In \emph{Advances in Neural Information Processing Systems (NIPS)}.
  MIT Press, 2003.

\bibitem[Shi et~al.(2015)Shi, Reid, Caetano, Van~den Hengel, and
  Wang]{shi2015hybrid}
Qinfeng Shi, Mark Reid, Tiberio Caetano, Anton Van~den Hengel, and Zhenhua
  Wang.
\newblock A hybrid loss for multiclass and structured prediction.
\newblock \emph{IEEE Transactions on Pattern Analysis and Machine
  Intelligence}, 2015.

\bibitem[Steinwart(2002)]{Steinwart2002}
Ingo Steinwart.
\newblock Support vector machines are universally consistent.
\newblock \emph{Journal of Complexity}, September 2002.

\bibitem[Stone(1977)]{stone1977consistent}
Charles~J. Stone.
\newblock Consistent nonparametric regression.
\newblock \emph{The Annals of Statistics}, 1977.

\bibitem[Tewari and Bartlett(2007)]{Tewari2007}
Ambuj Tewari and Peter~L. Bartlett.
\newblock On the consistency of multiclass classification methods.
\newblock \emph{Journal of Machine Learning Research}, 2007.

\bibitem[Zhang(2004{\natexlab{a}})]{Zhang}
Tong Zhang.
\newblock Statistical behavior and consistency of classification methods based
  on convex risk minimization.
\newblock \emph{The Annals of Statistics}, 2004{\natexlab{a}}.

\bibitem[Zhang(2004{\natexlab{b}})]{Zhang2004}
Tong Zhang.
\newblock Statistical analysis of some multi-category large margin
  classification methods.
\newblock \emph{Journal of Machine Learning Research}, 2004{\natexlab{b}}.

\end{thebibliography}
\end{document}